\documentclass{article}

\usepackage{natbib}
\setcitestyle{numbers,square,citesep={,}}

\usepackage{geometry}
\usepackage[utf8]{inputenc} 
\usepackage[T1]{fontenc}    
\usepackage[colorlinks=true,linkcolor=blue,urlcolor=blue,citecolor=blue,anchorcolor=blue]{hyperref}       
\usepackage{url}            
\usepackage{booktabs}       
\usepackage{amsfonts}       
\usepackage{nicefrac}       
\usepackage{microtype}      
\usepackage{xcolor}         

\usepackage{graphicx} 
\usepackage{amsmath}
\usepackage{amssymb}
\usepackage{mathtools}
\usepackage{pgfplots}
\usepackage{algorithm}%
\usepackage{algorithmic}%
\usepackage{xspace}
\usepackage{amsthm}
\usepackage{tikz}

\title{Formatting Instructions For NeurIPS 2024}

\author{%
  David S.~Hippocampus\thanks{Use footnote for providing further information
    about author (webpage, alternative address)---\emph{not} for acknowledging
    funding agencies.} \\
  Department of Computer Science\\
  Cranberry-Lemon University\\
  Pittsburgh, PA 15213 \\
  \texttt{hippo@cs.cranberry-lemon.edu} \\
}

\theoremstyle{plain}
\newtheorem{theorem}{Theorem}[section]

\newtheorem{lemma}[theorem]{Lemma}
\newtheorem{corollary}[theorem]{Corollary}
\theoremstyle{definition}

\theoremstyle{remark}

\newcommand{\hide}[1]{}

\def\R{\mathbb{R}}
\def\N{\mathbb{N}}

\def\M{\mathcal{M}}

\def\TV{\mathrm{TV}}

\def\diag{\mathrm{diag}}
\def\learnH{\tilde H}
\def\estH{\hat H}
\def\trueH{H}
\def\ones{\mathbf 1}
\def\uvec{e}
\def\zeros{\mathbf 0}
\def\hit{\ensuremath{\mathsf{hit}}\xspace}
\def\miss{\ensuremath{\mathsf{miss}}\xspace}

\def\ULTRAMC{\ensuremath{\mathsf{ULTRA\!-\!MC}}\xspace}

\renewcommand{\dagger}[0]{+}

\def\hit{\ensuremath{\mathsf{score}}\xspace}
\def\miss{\ensuremath{\mathsf{miss}}\xspace}

\newcommand{\spara}[1]{\paragraph{#1}}
\newcommand{\fabian}[1]{\textcolor{red}{[Fabian: #1]}}

\DeclareMathOperator{\vol}{vol}

\title{ULTRA-MC: A Unified Approach to Learning Mixtures of Markov Chains via Hitting Times}
\date{February 2024}
\author{Fabian Spaeh\\ fspaeh@bu.edu \\ Boston University \and Konstantinos Sotiropoulos \\ ksotirop@bu.edu \\ Boston University \and Charalampos E. Tsourakakis\\ ctsourak@bu.edu \\ Boston University}
\date{May 2023}

\begin{document}

\maketitle

\begin{abstract}
This study introduces a novel approach for learning mixtures of Markov chains, a critical process applicable to various fields, including healthcare and the analysis of web users. Existing research has identified a clear divide in methodologies for learning mixtures of discrete~\cite{gupta2016mixtures,spaeh2023casvd} and continuous-time Markov chains~\cite{spaeh2024www}, while the latter presents additional complexities for recovery accuracy and efficiency. 
 
We introduce a unifying strategy for learning mixtures of discrete and continuous-time Markov chains, focusing on hitting times, which are well defined for both types.
Specifically, we design a reconstruction algorithm that outputs a mixture which accurately reflects the estimated hitting times and demonstrates resilience to noise. We introduce an efficient gradient-descent approach, specifically tailored to manage the computational complexity and non-symmetric characteristics inherent in the calculation of hitting time derivatives.
Our approach is also of significant interest when applied to a single Markov chain, thus extending the methodologies previously established by \citet{hoskins2018learning} and \citet{wittmann2009reconstruction}. We complement our theoretical work with experiments conducted on synthetic and real-world datasets, providing a comprehensive evaluation of our methodology. 


\end{abstract}

\section{Introduction}
\label{sec:intro} 

Markov chains~\cite{levin2017markov} serve as a fundamental and remarkably versatile tool in modeling, underpinning a wide array of applications from Pagerank in Web search~\cite{page1999pagerank} to language modeling~\cite{lafferty2001document,mumford2010pattern}. 
A \textit{Markov chain} is a statistical model that describes a sequence of possible events in which the probability of new events depend only on the presence, but not the past. Formally, this characteristic is called the \textit{Markov property} and says that future states of the process depends only on the current state, not on the sequence of states that preceded it.
Markov chains are classified into \textit{Discrete-Time Markov Chains (DTMCs)}
and \textit{Continuous-Time Markov Chains (CTMCs)}.
The former involve a system making transitions between states at discrete time steps, with the probability of moving to the next state dependent on the current state. For the latter transitions occur continuously over time, and the system has a certain rate of transitioning from one state to another. 

Markov chains are used in various fields such as economics, game theory, genetics, and finance to model random systems that evolve over time depending only on their current state. While the assumption of Markovian dynamics simplifies many real-world situations \cite{chierichetti2012web},
they are a well-established tool that offers a robust mathematical
foundation with strong practical results. 

\begin{figure}
    \centering

\tikzstyle{court}=[black!20, line width=2pt,samples=100]
\tikzstyle{player}=[circle, inner sep=0pt, text width=16pt, align=center, draw=black!80, line width=1pt, fill=white]
\tikzstyle{pass}=[-latex, black]
\tikzstyle{basket}=[player, rectangle, inner sep=2pt, minimum height=13pt, text centered, text width=21pt]
\tikzstyle{start}=[player,fill=none,draw=blue,line width=2.5pt]

\scalebox{0.77}{
\begin{tikzpicture}[scale=1.6]

    \draw[court, domain=0:180] plot ({min(1.7, max(-1.7, 2*cos(\x)))}, {2*sin(\x)});
    \draw[court] (-0.5,0) -- (-0.5,1.2) -- (0.5,1.2) -- (0.5,0);
    \draw[court, domain=0:180] plot ({0.3*cos(\x)}, {1.2 + 0.3*sin(\x)});
    \draw[court, dashed, domain=0:180] plot ({0.3*cos(\x)}, {1.2 - 0.3*sin(\x)});
    \draw[court] (-1.8,0) -- (1.8,0);

    \node[player] (PG) at (0,2) {PG};
    \node[player] (SG) at (-1.3,1.5) {SG};
    \node[player] (PF) at (-1.0,0.6) {PF};
    \node[player] (C) at (0.9,1.0) {C};
    \node[player] (SF) at (1.1,0.3) {SF};

    \node[basket] (miss) at (-0.4,0) {$\mathsf{miss}$};
    \node[basket] (score) at (0.4,0) {$\mathsf{score}$};

    \node[label={[label distance=6pt]below:{1.7s}}] at (PF) {};
    \draw[pass,opacity=0.20847605885349677,line width=2.0847605885349676pt,green] (PF) to (score);
    \draw[pass,opacity=0.2001092564635487,line width=2.001092564635487pt] (PF) to (C);
    \node[label={[label distance=6pt]above:{5.0s}}] at (PG) {};
    \draw[pass,opacity=0.2845671598026733,line width=2.845671598026733pt,red] (PG) to (miss);
    \draw[pass,opacity=0.2743386814650957,line width=2.743386814650957pt] (PG) to (C);
    \node[label={[label distance=6pt]above:{3.4s}}] at (SG) {};
    \draw[pass,opacity=0.15267372181655098,line width=1.52673721816551pt,green] (SG) to (score);
    \draw[pass,opacity=0.2463312684467804,line width=2.463312684467804pt] (SG) to (C);
    \draw[pass,opacity=0.28734556401841704,line width=2.8734556401841704pt] (SG) to (SF);
    \node[label={[label distance=6pt]above:{1.5s}}] at (C) {};
    \node[start] at (C) {};
    \draw[pass,opacity=0.3468759216357968,line width=3.468759216357968pt] (C) to (PG);
    \draw[pass,opacity=0.32393812144070355,line width=3.2393812144070355pt] (C) to (SG);
    \node[label={[label distance=6pt]below:{4.0s}}] at (SF) {};
    \draw[pass,opacity=0.4009730291786204,line width=4.009730291786203pt,green] (SF) to (score);

\end{tikzpicture}}
\scalebox{0.77}{
\begin{tikzpicture}[scale=1.6]

    \draw[court, domain=0:180] plot ({min(1.7, max(-1.7, 2*cos(\x)))}, {2*sin(\x)});
    \draw[court] (-0.5,0) -- (-0.5,1.2) -- (0.5,1.2) -- (0.5,0);
    \draw[court, domain=0:180] plot ({0.3*cos(\x)}, {1.2 + 0.3*sin(\x)});
    \draw[court, dashed, domain=0:180] plot ({0.3*cos(\x)}, {1.2 - 0.3*sin(\x)});
    \draw[court] (-1.8,0) -- (1.8,0);

    \node[player] (PG) at (0,2) {PG};
    \node[player] (SG) at (-1.3,1.5) {SG};
    \node[player] (PF) at (-1.0,0.6) {PF};
    \node[player] (C) at (0.9,1.0) {C};
    \node[player] (SF) at (1.1,0.3) {SF};

    \node[basket] (miss) at (-0.4,0) {$\mathsf{miss}$};
    \node[basket] (score) at (0.4,0) {$\mathsf{score}$};

    \node[label={[label distance=6pt]above:{2.9s}}] at (C) {};
    \draw[pass,opacity=0.3450090989993528,line width=3.450090989993528pt,green] (C) to (score);
    \node[label={[label distance=6pt]above:{3.5s}}] at (SG) {};
    \draw[pass,opacity=0.1990736358941285,line width=1.990736358941285pt,green] (SG) to (score);
    \draw[pass,opacity=0.37499361565514855,line width=3.7499361565514855pt] (SG) to (PF);
    \node[label={[label distance=6pt]above:{4.1s}}] at (PG) {};
    \draw[pass,opacity=0.21957744173388472,line width=2.195774417338847pt,red] (PG) to (miss);
    \draw[pass,opacity=0.27189677177411714,line width=2.7189677177411715pt] (PG) to (SG);
    \draw[pass,opacity=0.23728026757180853,line width=2.372802675718085pt] (PG) to (PF);
    \node[label={[label distance=6pt]below:{2.4s}}] at (PF) {};
    \node[start] at (PF) {};
    \draw[pass,opacity=0.23577602764990205,line width=2.3577602764990204pt,green] (PF) to (score);
    \draw[pass,opacity=0.20992644658598875,line width=2.0992644658598874pt] (PF) to (C);
    \node[label={[label distance=6pt]below:{1.0s}}] at (SF) {};
    \draw[pass,opacity=0.2803720707066887,line width=2.803720707066887pt,green] (SF) to (score);

\end{tikzpicture}}
\scalebox{0.77}{
\begin{tikzpicture}[scale=1.6]

    \draw[court, domain=0:180] plot ({min(1.7, max(-1.7, 2*cos(\x)))}, {2*sin(\x)});
    \draw[court] (-0.5,0) -- (-0.5,1.2) -- (0.5,1.2) -- (0.5,0);
    \draw[court, domain=0:180] plot ({0.3*cos(\x)}, {1.2 + 0.3*sin(\x)});
    \draw[court, dashed, domain=0:180] plot ({0.3*cos(\x)}, {1.2 - 0.3*sin(\x)});
    \draw[court] (-1.8,0) -- (1.8,0);

    \node[player] (PG) at (0,2) {PG};
    \node[player] (SG) at (-1.3,1.5) {SG};
    \node[player] (PF) at (-1.0,0.6) {PF};
    \node[player] (C) at (0.9,1.0) {C};
    \node[player] (SF) at (1.1,0.3) {SF};

    \node[basket] (miss) at (-0.4,0) {$\mathsf{miss}$};
    \node[basket] (score) at (0.4,0) {$\mathsf{score}$};

    \node[label={[label distance=6pt]above:{2.7s}}] at (C) {};
    \node[start] at (C) {};
    \draw[pass,opacity=0.2861113525283822,line width=2.8611135252838222pt] (C) to (SF);
    \node[label={[label distance=6pt]above:{3.1s}}] at (SG) {};
    \draw[pass,opacity=0.17354568382518995,line width=1.7354568382518996pt,red] (SG) to (miss);
    \draw[pass,opacity=0.1798300380727732,line width=1.798300380727732pt,green] (SG) to (score);
    \draw[pass,opacity=0.2454367258864194,line width=2.454367258864194pt] (SG) to (PG);
    \node[label={[label distance=6pt]above:{4.7s}}] at (PG) {};
    \draw[pass,opacity=0.301043436315341,line width=3.0104343631534096pt,red] (PG) to (miss);
    \node[label={[label distance=6pt]below:{1.9s}}] at (PF) {};
    \draw[pass,opacity=0.24080843737214616,line width=2.4080843737214614pt] (PF) to (SG);
    \node[label={[label distance=6pt]below:{1.2s}}] at (SF) {};
    \draw[pass,opacity=0.5022308425772863,line width=5.022308425772863pt,red] (SF) to (miss);

\end{tikzpicture}}

    \caption{Three offensive strategies of the Denver Nuggets during the 2022 season, learned from a mixture of $C=6$ continuous-time Markov chains.
    We provide the remaining strategies in Figure~\ref{fig:nuggets}
    and a detailed explanation in Section~\ref{sec:exp}.}
    \label{fig:nba}
\end{figure}
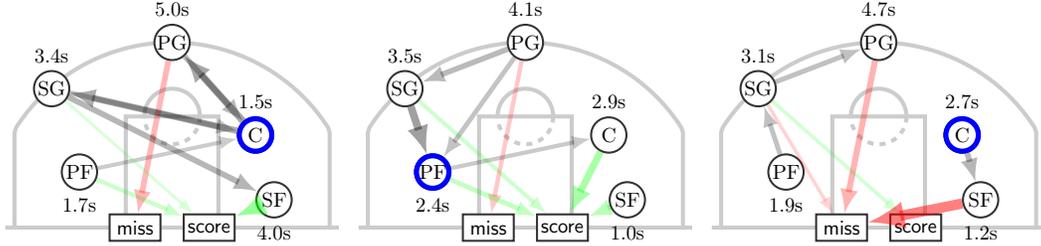

To extend the applicability of Markov chains beyond situations
that are approximately Markovian, researchers and practitioners employ
mixtures of Markov chains.
Figure~\ref{fig:nba} shows an example, where we model multiple strategies
of a basketball team,
assuming that each individual strategy plays out as approximately Markovian.
Another application lies in discrete choice where we study user
preferences for choosing between multiple items, for instance in online shops.
Single Markov chains can be used to learn a restricted model of discrete
choice~\cite{blanchet2016markov} while mixtures are known to
enhance their modeling power~\cite{chierichetti2018learning}.
There are many more applications, some
in ecology or cancer research~\cite{patterson2009classifying,fearon1990genetic}.

Despite the extensive research on unmixing various distributions
\cite{dasgupta1999learning,gordon2021identifying},
the specific challenge of deciphering mixtures of Markov chains has not been as extensively explored. \citet{gupta2016mixtures} introduce a seminal reconstruction algorithm that is provably efficient and based on the Singular Value Decomposition (SVD). Spaeh and Tsourakakis~\cite{spaeh2023casvd} extend this work by relaxing some of the necessary conditions of their work.
Until recently~\cite{spaeh2024www}, the task of learning mixtures of continuous time Markov chains (CTMCs) had not garnered significant attention. Although recent advancements have offered practical tools for learning different mixtures of CTMCs, complete with theoretical assurances and polynomial-time algorithms, these methods significantly differ from one another. In our study, we present a methodology that is independent of whether the data originates from a discrete or a continuous time Markov chain.
Instead, it is based on \emph{hitting times}, a statistic that is common to both settings. This allows us to make the following contributions:

\begin{enumerate}
\item We show how to learn a single Markov chain using (a noisy subset of) hitting times via projected gradient descent. 
The main contribution lies in formulating an efficient expression for the
computation of gradients. Unlike effective resistances~\cite{hoskins2018learning,wittmann2009reconstruction}, hitting times are asymmetric (meaning the hitting time from node $u$ to node $v$ can significantly vary from that from $v$ to $u$) introducing complications in deriving efficient gradient formulations. Indeed, we show that our approach significantly surpasses naive evaluation and numerical estimation and improves performance by several orders of magnitude.  

\item We integrate our algorithm for learning a single chain within an expectation maximization (EM) approach to learn a mixture of Markov chains. This method is applicable for both discrete and continuous time Markov chains. Due to the efficiency of our approach, this allows us to scale to over 1000 nodes, while previous approaches were limited to $\approx 100$ nodes \cite{spaeh2023casvd,gupta2016mixtures} in reasonable time.

\item Our research includes a wide range of experiments designed to empirically address key questions. These include assessing our ability to deduce hitting times from trails, the effectiveness of our method to learn a single Markov chain, and the proficiency in learning a mixture. 
\end{enumerate}
We further apply our algorithm to \emph{Markovletics} \cite{spaeh2024www} where we learn multiple offensive strategies from the passing game of NBA teams. Figure~\ref{fig:nba} shows a preview of our results. Specifically, it illustrates three strategies learned using our algorithm \ULTRAMC for the Denver Nuggets during the 2022 season. 
The role of the Center player (C), a position held by NBA MVP for the 2020–21 and 2021–22 seasons Nikola Jokic, is evidently crucial as a passer and a successful scorer.








\section{Preliminaries}
\label{sec:prelim}

\spara{Discrete and Continuous-Time Markov Chains}
A discrete-time Markov chain (DTMC) is defined
by a stochastic matrix $M \in \R^{n \times n}$
of transition probabilities over the
state space $[n] = \{1, \dots, n\}$. A random
walk from state $u \in [n]$ unfolds
as follows: We randomly transition to the next
state $v$ with probability $M_{uv}$, then repeat
this from state $v$. Observing the random
walk yields a discrete-time trail $\mathbf x$,
where $\mathbf x_t \in [n]$ is the state at step
$t \in \N_0$.
An undirected graph $G$ with vertices $[n]$
corresponds to a Markov chain with
$M_{uv} = \frac 1 {\mathrm{deg_G(u)}}$
if $u$ and $v$ are neighbors and $M_{uv} = 0$,
otherwise.

In a continuous-time Markov chain (CTMC),
transitions can happen at any time and the
rate of transition is given by a rate
matrix $K \in \R^{n \times n}$. In a random
walk from $U$, we now sample exponential-time
random variables $E_v \sim \mathrm{Exp}(K_{uv})$
for all states $v \not= u$. The next state is
$v^* = \arg\min_{v} E_v$ and the transition
occurs after a time of $E_{v^*}$. We repeat this
process from $v^*$. Observing the random walk
yields a  continuous-time trail $\mathbf x$
where $\mathbf x_t \in [n]$ denotes the
state at time $t \in \R_{\ge 0}$.

The hitting time for a given Markov chain is
defined for each pair of nodes $u, v \in [n]$.
It is the expected time to reach
$v$ from $u$ in a random walk, defined as
$h(u, v) = \mathbb E_{\mathbf x}[\min\{t \ge 0 : \mathbf x_t = v\} \mid \mathbf x_0 = u]$.
We define the hitting time matrix
$H \in \R^{n \times n}_{\ge 0}$
through $H_{uv} = h(u,v)$.
The commute time $c(u, v)$ extends this notion and
is defined as the expected time to go from $u$
to $v$ and back to $u$. It is thus
$c(u, v) = h(u, v) + h(v, u)$.
For random walks on symmetric Markov chains,
the effective resistance $r(u, v)$ between
nodes $u$ and $v$ scales the commute time such
that $r(u, v) \propto c(u, v)$.
The effective resistance itself can be directly
calculated through the Moore-Penrose pseudoinverse of
the combinatorial Laplacian~\cite{spielman2012spectral}.
Due to this connection, commute times are
much better understood.

\spara{Mixtures of Markov Chains}
A mixture of $C$ discrete-time Markov chains
is defined as a tuple
of stochastic matrices
$\mathcal M = (M^1, M^2, \dots, M^C)$.
A mixture of $C$ continuous-time Markov chains
is defined as a tuple of rate matrices
$\mathcal M = (K^1, K^2, \dots, K^C)$.
In both cases,
each chain is associated
with a vector of starting probabilities
$\alpha^i$ for each chain $i \in [C]$
such that
$\sum_{i = 1}^C \sum_{u=1}^n \alpha^i_u = 1$.
We generate trails as follows: We sample
a chain $i \in [C]$ and state $u \in [n]$ with
probability $\alpha^i_u$. We then generate
a trail starting from $u$ according to the
$i$-th Markov chain.

\begin{table}
  \centering
  \caption{  \label{tab:symbols}
    Frequently used notation.}
  \medskip
  \begin{tabular}{ll}
    \toprule
    Symbol & Definition \\
    \midrule
    $[n]$ &  state space $[n] = \{1, \dots, n\}$ \\
    $C \in \N$ & number of chains \\
    $M \in [0,1]^{n \times n}$ & discrete-time (DT)  transition matrix \\
    $K \in \R^{n \times n}$ & continuous-time (CT)  rate matrix \\
    $L = I - M$ & Laplacian of transition matrix $M$ \\
    $\trueH \in \R^{n \times n}$ & hitting time matrix: $H_{uv} = h_G(u, v)$ \\
    $\estH \in \R^{n \times n}$ & estimated hitting times (i.e., approximation of $\trueH $ from the trails) \\
    $\learnH \in \R^{n \times n}$ &  hitting times  of learned mixture \\
    $\tau > 0$ & discretization rate for CTMCs \\ 
    \bottomrule
  \end{tabular}
\end{table}

\spara{Notation}
%
Let $e_u$ denote the $u$-th standard
basis vector and
$\chi_{uv} = e_u - e_v$.
The Hadamard product
for matrices $X \in \R^{n \times n}$ and $Y \in \R^{n \times n}$, 
is defined as $(X \circ Y)_{uv} = X_{uv} \cdot Y_{uv}$.
We denote
the Moore-Penrose pseudoinverse
of $X$ as $X^+$.
We further use the tensor product
$\frac{\partial X}{\partial Z} =
 \frac{\partial X}{\partial Y} \otimes
 \frac{\partial Y}{\partial Z}$
for the chain
rule in multiple
dimensions~\cite{deisenroth2020mathematics}.
Other notation along with additional details is summarized in Table~\ref{tab:symbols}.

\section{Related work}
\label{sec:related} 

\spara{Learning a Single Markov Chain}
%
Wittmann et al.~\cite{wittmann2009reconstruction}
reconstruct a discrete-time Markov
chain from the complete hitting time matrix
by solving a linear system.
Their work does not consider reconstruction under
a subset of noisy hitting times and has no
counterpart for continuous time.
Nonetheless, we use their work to obtain an initial
estimate for our gradient descent scheme.
\citet{cohen2016} derive an expression for the hitting times for a discrete-time Markov chain (cf. Equation~\eqref{eq:1}). We adopt the same assumption as \citet{cohen2016} on the existence of the stationary distribution. 
A related problem is the reconstruction from all pairwise commute times. The access to hitting times allows for the easy computation of the commute times, thereby simplifying the problem. \citet{hoskins2018learning} reconstruct a Markov chain from an incomplete set of pairwise commute times, which may also include noise. They optimize the convex relaxation of a least squares formulation. Zhu et al. propose a   low-rank optimization to learns a Markov chain from a single trajectory of states~\cite{zhu2022learning}.

The problem of reconstructing a Markov chain using only a partial set of noisy hitting times has not yet been studied in the existing literature. As we will explore, the task of reconstructing a Markov chain from hitting times presents significant challenges, primarily due to their asymmetric nature. 

\spara{Learning Mixtures of Discrete-Time Markov Chains (DTMCs)}
%
Despite the extensive research on unmixing distributions~\cite{subakan2013probabilistic,lindsay1995mixture,sanjeev2001learning,dasgupta1999learning,gordon2021identifying} and the clarity of the problem in recovering a discrete MC mixture, the latter has received less attention.  The prevalent method for unmixing in this scenario is the expectation maximization (EM) algorithm~\cite{dempster1977maximum}, which alternates between clustering the trails to chains and learning the parameters of each chain.
EM may be slow and does not offer recovery guarantees, but proves
useful when applied in the right context~\cite{spaeh2023casvd,gupta2016mixtures}.
Learning a mixture of Dirichlet distributions~\cite{neal2000markov} suffers from similar drawbacks~\cite{subakan2013probabilistic}.
Techniques based on moments, utilizing tensor and matrix decompositions, have been demonstrated to effectively learn, under certain conditions, a mixture of hidden Markov models~\cite{anandkumar2012method,anandkumar2014tensor,subakan2014spectral,subakan2013probabilistic}  or Markov chains~\cite{subakan2013probabilistic}.

\citet{gupta2016mixtures} introduce a singular value decomposition (SVD) based algorithm which achieves exact
recovery under specific conditions without noise.
Spaeh and Tsourakakis~\cite{spaeh2023casvd} extend their method to addresses some of the limitations, such as the connectivity of the chains in the mixture.
\citet{kausik2023mdps} demonstrate a method for learning a mixture when the trails have minimum length $O(t_\mathrm{mix})$, where $t_\mathrm{mix}$ is the longest mixing time of any Markov chain in the mixture. \citet{spaeh2024www} study the impact of trail lengths on learning mixtures.

\spara{Learning Mixtures of Continuous-Time Markov Chains (CTMCs)}
%
\citet{tataru2011comparison} evaluate various methods, revealing that these approaches primarily calculate different weighted linear combinations of the expected values of sufficient statistics.
\citet{mcgibbon2015mle} develop an efficient maximum likelihood estimation (MLE) approach for reconstructing a single CTMC from sampled data, which works
on discretized continuous-time trails.
The challenge of learning a mixture of CTMCs has only recently been addressed by \citet{spaeh2024www}. They propose several discretization-based approaches
within a comprehensive framework characterizing different problem regimes based on the length of the trails.

\spara{Choosing the Number of Chains $C$} The task of determining the optimal number of latent factors in dimensionality reduction has been extensively studied, leading to various proposed methods
 \cite{donoho2013optimal,zhu2006automatic,jolliffe2016principal,efron1991statistical,kodinariya2013review,suhr2005principal}.
In the context of learning mixtures DTMCs, \citet{spaeh2023casvd} propose a criterion based on the singular values of certain matrices originally introduced by Gupta et al.~\cite{gupta2016mixtures} to determine $C$. In this paper, we will assume $C$ is part of the input for the theory.

\section{Proposed method} 
\label{sec:proposed}


We now describe our approach to
inferring Markov chains from
the hitting times $\trueH \in \R^{n \times n}$.
Beyond the intellectual fascination of deriving even a single Markov chain, as explored in prior studies on commute times~\cite{hoskins2018learning,wittmann2009reconstruction}, it is evident that hitting times offer a robust analytical tool for applications where there is a significant asymmetry from one state to another. Such applications naturally arise in epidemiology or ecology where animal movement from habitat to habitat~\cite{patterson2009classifying}, but also in cancer research where driver mutations are likely to precede other mutations~\cite{fearon1990genetic}.

In Section~\ref{sec:prop-single}, we introduce
a projected gradient descent
approach that iteratively
improves the pseudoinverse
of the Laplacian to match
the given hitting times $\trueH$.
To this end, we derive
an efficient analytical
expression of the 
gradients which vastly
outperforms a naive
implementation through
automatic differentiation
(Autodiff) via the
chain rule.
In Section~\ref{sec:prop-mixture},
we describe
how our approach can be
applied to learn mixtures
of Markov chains from
trails by
estimating the hitting
times and applying an
EM-style
algorithm.
This is the
first algorithm to learn
mixtures of discrete-time
and continuous-time
in a unified way.
We are furthermore able to
compete or outperform
state of the art algorithms
in both settings,
which we empirically
demonstrate in
Section~\ref{sec:exp}.

Throughout this paper, we assume that chains are both irreducible (i.e. strongly connected) and aperiodic. This ensures that the stationary distribution is well defined.
Nevertheless, in Section~\ref{sec:exp}, we explore how to transcend this assumption in practical settings and learn directed acyclic graphs.

\subsection{Learning a Single Markov Chain from Hitting Times}
\label{sec:prop-single}

We start with
the following problem:
how do we recover a single Markov chain with
transition matrix $M \in \R^{n \times n}$
from a matrix of known or estimated
hitting times $\trueH$?
In our approach,
we use projected gradient descent
where we iteratively improve our
estimate of $M$.
For convenience and computational
efficiency, we
choose to optimize over
the pseudoinverse of the Laplacian $L^+$.
Our approach executes alternating
gradient descent steps and projections.
In particular, we project onto the
set of Laplacian pseudoinverses that
correspond to discrete or continuous-time
Markov chains.

We us to use the
following expression given in \cite{cohen2016}
for the hitting times
from   \( u \) to   \( v \):
\begin{equation}
    \label{eq:1}
    H_{uv} = H_{uv}(L^+) = \bigg(\mathbf{1}-\frac{1}{s_v}\mathbf{e}_v\bigg)^{\top}L^{+}\chi_{vu}
\end{equation}
Here, \( s \in \mathbb R^n \) represents the stationary
distribution of the Markov chain,
which we compute using Lemma~\ref{lem:stationary}
below.
We show how to derive an analogous statement
for a CTMC with rate matrix $K$
in Lemma~\ref{lem:ht-continuous} in Appendix~\ref{sec:apx-omitted}. 
Here, the negative rate matrix $-K$ takes
the place of the Laplacian in Equation~\ref{eq:1}.
This is fundamental to our unifying approach.
The following statements thus hold true
simultaneously for CTMCs where, in an abuse
of notation, we use
$L = -K$.

The concrete goal of our
gradient descent approach
is to improve an estimate of
the Laplacian pseudoinverse $L^+$
such that $\trueH \approx H(L^+)$
where we compute the hitting times
according to Equation~\eqref{eq:1}.
We do this by minimizing the $\ell_2$-loss over
the learned hitting times $\learnH = H(L^+)$
given as
\[
    \ell_2(L^+) = \frac 1 2 \sum_{u, v} ( \learnH_{uv} - \trueH_{uv} )^2
    = \frac 1 2 \| \learnH - \trueH \|_F^2.
\]
The main difficulty 
is to compute the gradients $\nabla \ell_2(L^+)$
efficiently.
By the chain rule,
\[
    \nabla \ell_2(L^+)
    = \bigg( \frac{\partial \learnH}{\partial L^+} \bigg)^\top
    \otimes (\learnH - \trueH) .
\]
Note that
$\frac{\partial \learnH}{\partial L^+} \in \R^{n \times n \times n \times n}$ 
which suggests that computing the derivative naively
has complexity $\Omega(n^4)$.
However, we are able to derive
an efficient expression
of the gradients which
shows that it is
possible to reduce this complexity to $O(n^\omega)$
where $\omega \approx 2.37$ is the matrix multiplication constant.
We now outline
our approach to determine $\nabla \ell_2(L^+)$
efficiently, but defer all proofs to Appendix~\ref{sec:apx-omitted}.
We begin by splitting the computation
of $\nabla \ell_2(L^+)$ into
two components which we
evaluate separately.

\begin{lemma}
\label{lem:split}
We have
$H = A - B$
for
$a=\ones^{\top}L^{+}\in\R^{1\times n}$
and
\begin{align*}
    A &= a^{\top}\ones^{\top}-\ones a \\
    B &= \left(L^{+}-\diag\left(L^{+}\right)\ones^{\top}\right)^{\top}\diag\left(s\right)^{-1} .
\end{align*}
\end{lemma}
We can now write
\begin{align}
    \label{eq:6}
    \left( \frac{\partial H}{\partial L^+} \right)^\top
    \otimes \Delta
    = 
    \left( \frac{\partial A}{\partial L^+} \right)^\top
    \otimes \Delta
    -
    \left( \frac{\partial B}{\partial L^+} \right)^\top
    \otimes \Delta
\end{align}
where $\Delta = \learnH - \trueH$.
We evaluate both tensor products separately.
The first term follows by
a relatively straightforward calculation:

\begin{lemma}
\label{lem:gradA}
The term $\left( \frac{\partial A}{\partial L^+} \right)^\top \otimes \Delta$ equals
$$ \ones\left(\ones^{\top}\Delta\right)-\left(\ones^{\top}\Delta\right)^{\top}\ones^{\top}-\ones\left(\Delta\ones\right)^{\top}+\left(\Delta\ones\right)\ones^{\top} .$$
Furthermore,
$\left( \frac{\partial A}{\partial L^+} \right)^\top \otimes \Delta$
can be computed in  time $O(n^2)$.
\end{lemma}

The
computation of $\learnH$
and the second tensor
product involves the (inverse
of) the stationary distribution $s$,
for which we derive
the following expression.
\begin{lemma}
\label{lem:stationary}
The stationary distribution for a Markov chain with
Laplacian $L$ is $s=\frac{d}{\| d \|_1}$ where $d=\ones - L L^+ \ones$.
\end{lemma}
Recall that we assume throughout the paper
that the Markov chains we consider are aperiodic
and irreducible, so the stationary
distribution is unique and well defined. 
%
To evaluate the derivative
of the stationary distribution $s$,
we need the
derivative of the Laplacian $L$
in terms of its pseudoinverse $L^+$.
We derive an expression for this in Appendix~\ref{sec:apx-omitted}
and use this for the
following lemma, which
describes an efficient expression
for the second tensor product
in Equation~\eqref{eq:6}.


\begin{lemma}
\label{lem:gradB}
Let 
$D \in \R^{n\times n}$ be a diagonal matrix with
$D_{ww} = \frac{1}{d_{w}}\left(\frac{1}{s_{w}}\uvec_{w}^{\top}-\ones^{\top}\right)$.
Define the vectors
$g_{uv} = \uvec_u\chi_{vu}^{\top}(LL^{+} + I )\ones \in \R^n$
for all $u, v \in [n]$ and
the matrix
$G \in \R^{n \times n}$ with
\[
    G_{uv} = \ones^{\top}\left(\left(\left(L^{+}-\diag\left(L^{+}\right)\ones^{\top}\right)^{\top}D\right)\circ\Delta\right) g_{uv} .
\]
We have
\[
    \left( \frac{\partial B}{\partial L^+} \right)^\top \otimes \Delta
    = \diag(s)^{-1} \ones \ones^\top (\Delta^\top - \Delta)
     + \Delta \diag(s)^{-1} \ones \ones^\top + G .
\]
Furthermore, 
$\left( \frac{\partial B}{\partial L^+} \right)^\top \otimes \Delta$ can be computed in time
$O(n^\omega)$.
\end{lemma}

These gradient computations
trivially extend to a weighted loss
$\ell_2(L^+; W) = \frac 1 2 \| W \circ (\tilde H - H) \|_F^2$
where $W \in \R^{n \times n}_{\ge 0}$.
We do not explore a weighted loss further,
but this may be useful as we can set
$W_{uv} = 1$ only if we observed a trail
between $u$ and $v$ and $W_{uv} = 0$,
otherwise. Another choice is to set
$W_{uv}$ inverse proportional to the
sample variance, similar
to generalized least squares.

Even though our
loss is non-convex,
we are able to minimize $\ell_2(L^+)$
reliably, which we show
empirically in Section~\ref{sec:exp}.
In the following, we
show how to  incorporate our
gradient descent approach into
an algorithm for learning mixtures
of Markov chains.
Here, the efficiency of our gradient
calculations 
are key to obtaining a new algorithm
that scales well in the number
of nodes $n$, which was
challenging for previous
works~\cite{gupta2016mixtures,spaeh2024www}.

\subsection{\ULTRAMC : Learning a Mixture of Markov Chains from Trails}
\label{sec:prop-mixture}

To learn mixtures,
we follow an expectation-maximization (EM) style
approach. 
Here, we critically use that our algorithm
for learning Markov chains from hitting
times is robust against noise. Noise naturally
arises during expectation-maximization from
inaccurate soft clusterings. For instance,
the initial soft clustering is completely random.
We detail our method in Algorithm~\ref{alg:ht-em}
for the discrete-time setting, but this trivially
extends to continuous-time case.
\begin{algorithm}[h]
   \caption{\ULTRAMC for learning mixtures of Markov chains} 
   \label{alg:ht-em}
\begin{algorithmic}[1]
   \STATE {\bfseries Input:} Set of trails $\mathbf X$, number of chains $C$
   \STATE {\bfseries Output:} Mixture $\mathcal M = (M^1, \dots, M^C)$
   \STATE Initialize $\mathcal M = (M^1, \dots, M^C)$ randomly
   \WHILE{$\mathcal M$ has not converged}
        \FOR{$i = 1, \dots, C$}
           \STATE Let $p_{i,\mathbf x} \gets
                \frac{\Pr[x \mid M^i]}{\sum_{j} \Pr[x \mid M^{j}]}$
                for each $\mathbf x \in \mathbf X$
           \STATE Estimate $\estH^i$ from $\mathbf X$ using weights $p_{i,\mathbf x}$
           \STATE Learn $M^i$ from $\estH^i$ through projected gradient descent
        \ENDFOR
   \ENDWHILE
\end{algorithmic}
\end{algorithm}


In our approach \ULTRAMC,
we iteratively refine a guess
to the true mixture
$\mathcal M = (M^1, \dots, M^C)$.
In each iteration, we estimate the
likelihood of each trail
$\mathbf x \in \mathbf X$ to
belong to a chain $M^i$ for $i \in [C]$.
Specific to our approach is
that we use these likelihoods
as weights to estimate the
hitting time matrix $\estH^i$ for each
chain $i \in [C]$.
Finally, we use these hitting
times to recompute
all Markov chains $M^i$
using the gradient descent
approach of
Section~\ref{sec:prop-single}.
Our exact approach to estimate
hitting times is detailed
in Algorithm~\ref{alg:ht-est}
in Appendix~\ref{sec:apx-omitted},
along a discussion on bias and
variance of the estimation.

\section{Experimental results} 
\label{sec:exp}
In this section, we conduct an empirical assessment of the algorithms introduced in this study. We demonstrate improved accuracy of our method for learning a single Markov chain (Section~\ref{sec:prop-single}) compared to the prior work.
We also illustrate that \ULTRAMC (Section~\ref{sec:prop-mixture}) is able to meet or surpass the performance of existing algorithms for learning mixtures of Markov chains both in terms of efficiency and accuracy. It is worth mentioning that developing more efficient methods is an open research direction; our current approach is practical for constant values of C and up to 1000 states, with a runtime of a few hours.

\subsection{Experimental setup} 

\spara{Synthetic datasets} We use four graph types: the complete graph $K_n$, the star graph $S_n$, the lollipop graph ({\small \textsf{LOL}$_n$}), and the square grid graph
$G_{\sqrt{n},\sqrt{n}}$. In the lollipop graph, 
there is a complete graph $K_{n/2}$ connected to a path $P_{n/2}$.
We choose $n=16$ 
which is sufficient to demonstrate challenges in our settings.
The four graphs exhibit different hitting time distributions
and the largest hitting time between any pair of nodes
(known as the cover time)
also varies significantly \cite{sericola2023hitting,rao2013finding};
notably, the lollipop graph is known to have the
maximum cover time among all
unweighted, undirected graphs~\cite{brightwell1990maximum,feige1995tight}.
We obtain a Markov chain from these
graphs by choosing the next state
uniformly among the neighborhood
of the current state.
We also generate uniform random Markov chains.
In the discrete-time setting, every possible transition is
allowed and its weight is chosen uniformly
at random from $[0,1]$.
These uniform weights are then normalized to form a valid stochastic
transition matrix $M$. 
In the continuous setting, we generate random rate matrices. Here, we
choose the rate $K_{uv}$ uniformly and independently at random
from the interval $[0, 1]$ for each pair of states $u\not= v$.
To obtain mixtures,
we combine multiple
such random Markov chains
as done in the prior work
\cite{gupta2016mixtures, spaeh2024www}





\spara{NBA dataset} 
We use a dataset from Second Spectrum~\cite{secondspectrum} on the
passing game in the NBA from which
we generate continuous-time passing sequences for each team. In these sequences, every player is represented as a state, with state transitions occurring whenever the ball is passed between players. These sequences effectively capture offensive plays, concluding when the ball is taken over by the opposing team or a shot attempt is made. To denote the outcome of each offensive play, we use two additional states, $\hit$ and $\miss$ depending on whether the team scored or not. This dataset encompasses the 2022 and 2023 NBA seasons and includes a total of 1\,433\,788 passes. These passes are part of 535\,351 offensive opportunities from 2460 games. We only consider trails with a duration between 10 and 20 seconds. We generate an average of 3850 sequences per team. 

\spara{Methods}
To estimate a single Markov chain
from either true or estimated hitting times,
we use the efficient gradient descent
approach of Section~\ref{sec:prop-single}, which
we also refer to as \ULTRAMC.
In our implementation, we use
the ADAM optimizer \cite{kingma2016}
with parameters $\beta_1=0.99$
and $\beta_2=0.999$
and a learning rate of $\eta=10^{-4}$.
We start either with a random Markov chain,
or we use the output of \citet{wittmann2009reconstruction}
projected onto the feasible set.
We denote the latter as \textsf{WSBT}.
We estimate hitting times
via Algorithm~\ref{alg:ht-est}
in Appendix~\ref{sec:apx-exp}.
To learn a mixture of Markov chains,
we use \ULTRAMC as
introduced in Section~\ref{sec:prop-mixture}
limited to 100 iterations.
Our algorithms work for discrete
and continuous-time Markov chains.
We consider baselines
from prior works:
For discrete-time Markov chains,
we use a basic expectation maximization
{\small \textsf{EM (discrete)}}
and the SVD-based approach
{\small \textsf{SVD (discrete)}} \citet{gupta2016mixtures}.
{\small \textsf{SVD (discrete)}} works on
trails of length three, which
we can obtain from our instances
by subdividing each trail.
For continuous-time Markov chains,
we use methods detailed in
\cite{spaeh2024www}.
These are the discretization-based
methods 
{\small \textsf{SVD (discretized)}},
{\small \textsf{EM (discretized)}},
{\small \textsf{KTT (discretized)}}~\cite{kausik2023mdps},
where we use $\tau=0.1$ for the discretization rate,
and continuous-time expectation maximization
{\small \textsf{EM (continuous)}}.
We stop all EM-based methods after
100 iterations or convergence.

\spara{Metrics}
To compare a learned mixture to the
ground truth,
we use the
recovery error \cite{gupta2016mixtures}
which is defined via the total variation
(TV) distance. For discrete-time,
we can compute
the TV distance through
$
    \TV(M_u, M'_u) =
        \frac 1 2 \sum_{v=1}^n | M_{uv} - M'_{uv}|
$
and we use a similar formulation
for continuous-time
\cite{spaeh2024www}.
We also use this expression when the
output is not a stochastic matrix, which
may be the case in the method
of \citet{wittmann2009reconstruction} under noise.
For two discrete or continuous-time
Markov chains $M$ and $M'$,
the recovery error is the average
TV-distance
\[
    \textrm{recovery-error}(M, M') = \frac 1 n \sum_{u=1}^n \TV(M_u, M'_u) .
\]
Finally, for a mixture of $C$ Markov
chains, the recovery error is defined
as the average recovery error of the
best assignment, i.e.
\[
    \textrm{recovery-error}(\mathcal M, \mathcal M')\\ =
    \frac 1 C \min_{\sigma \in S_C} \sum_{i = 1}^C
        \textrm{recovery-error}(M^i, (M')^{\sigma(i)})
\]
where $S_C$ is the group of all
permutations on $[C]$.
We compare true and estimated hitting
times with the Frobenius
norm $\| \estH - \trueH \|_F$ defined as
$
    \| \estH - \trueH \|_F^2 =
    \sum_{u, v} (\estH_{uv} - \trueH_{uv})^2
$.

\spara{Code} Our Python and Julia code is available online.\footnote{\url{https://github.com/285714/HTInference}}
We executed our code on a 2.9 GHz Intel Xeon Gold 6226R processor using 384GB RAM.

\subsection{Experimental Findings}
\label{sec:findings}

\begin{figure}[t]
    \centering
    
    \def\fact{0.25}
    \def\negspace{-1em}
    \hspace{-2em}~
    \includegraphics[width=\fact\textwidth]{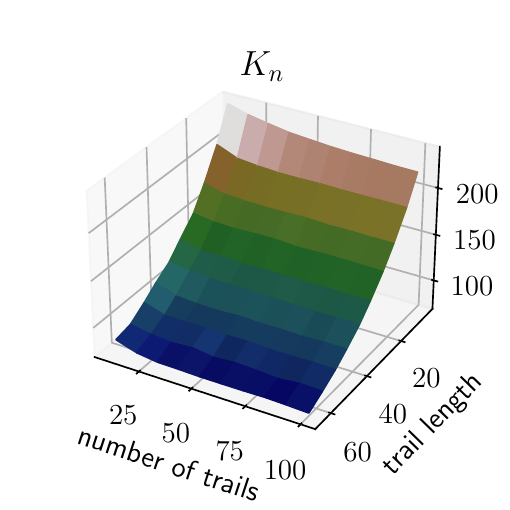}~
    \hspace{\negspace}~
    \includegraphics[width=\fact\textwidth]{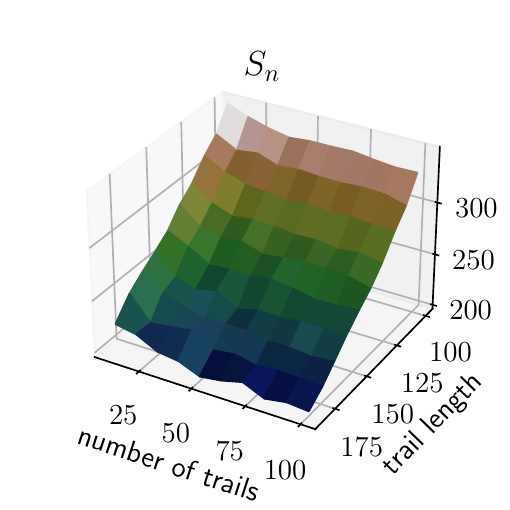}~
    \hspace{\negspace}~
    \includegraphics[width=\fact\textwidth]{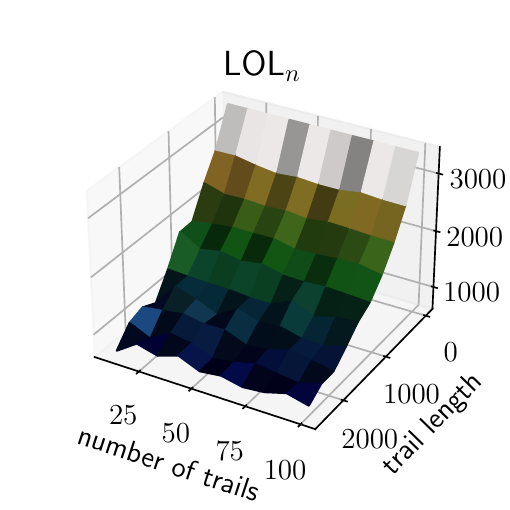}~
    \hspace{\negspace}~
    \includegraphics[width=\fact\textwidth]{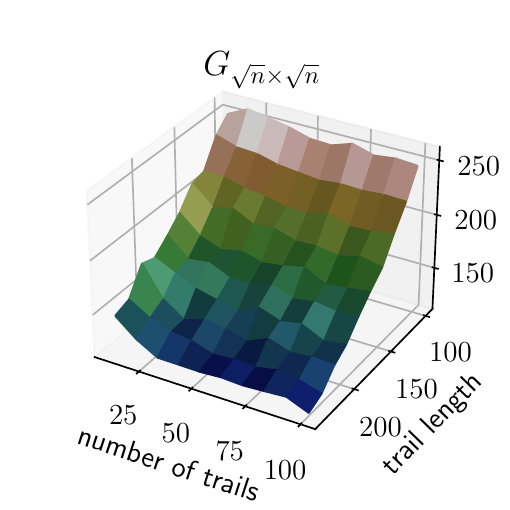}~
    
    \caption{We measure the error $\| \trueH - \estH \|_F$ in the estimation of hitting
    times from trails for $n=16$ nodes on the complete graph ($K_n$), star graph ($S_n$),
    lollipop graph ({\scriptsize \textsf{LOL}}), and grid graph ({\scriptsize \textsf{GRID}}).
    We show the mean over $5$ runs.
    The cover time for the graphs are $16$, $46$, $612$, and $\approx 59.4$ respectively
    and we vary the trail length to a multiple of the cover time for each graph.
    }
    \label{fig:ht-est}
\end{figure}

\begin{table}[t]
\caption{Running times of different gradient implementations,
over 1000 gradient descent iterations.}
\label{tab:runtime}
\begin{center}
\begin{small}
\begin{tabular}{r|ccccc}
\toprule

$n$ &
$5$ &
$10$ &
$50$ &
$1000$ &
$2000$ \\
\midrule

Analytical &
$0.03\mathrm s \, \pm 0.01$ &
$\phantom{10}0.28\mathrm s \, \pm 0.01$ &
$1.41\mathrm s \, \pm 0.06$ &
$10.07\mathrm{m} \, \pm 0.19$ &
$76.33\mathrm{m} \, \pm 0.18$ \\

Autodiff &
$0.40\mathrm s \, \pm 0.02$ &
$100.40\mathrm s \, \pm 1.78$ &
$>2\mathrm h$ &
$>2\mathrm h$ &
$>2\mathrm h$ \\

Numerical &
$2.64\mathrm s \, \pm 0.10$ &
$164.67\mathrm s\, \pm 1.80$ &
$>2\mathrm h$ &
$>2\mathrm h$ &
$>2\mathrm h$ \\

\bottomrule
\end{tabular}
\end{small}
\end{center}
\end{table}

\begin{figure}
    \centering
    \includegraphics[width=0.5\columnwidth]{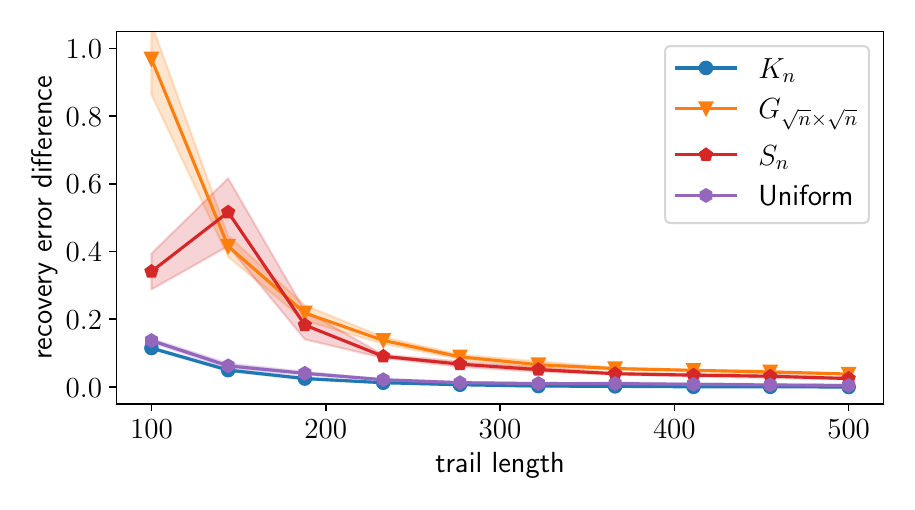}~
    \includegraphics[width=0.5\columnwidth]{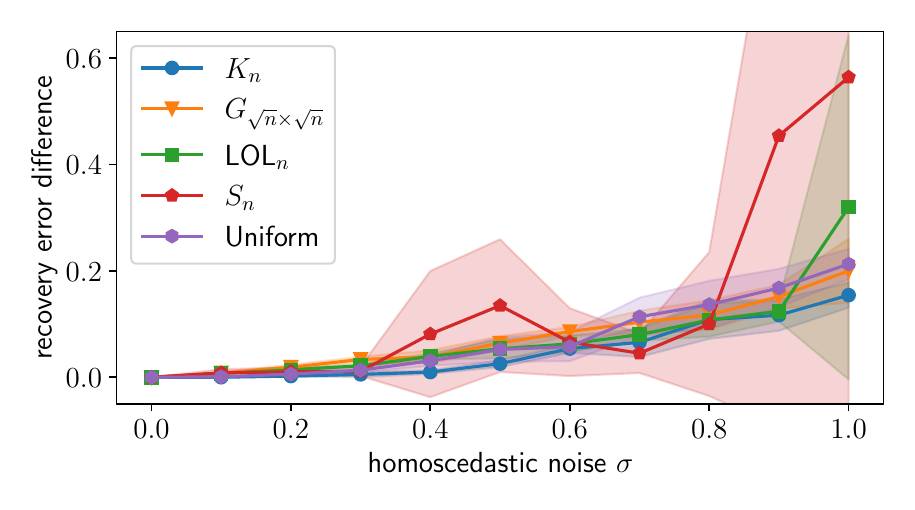}
    \vspace{-2em}
    \caption{Learning Markov chains from noisy hitting times, for different
    graph types and noise levels.  We plot the difference between the recovery error achieved by \citet{wittmann2009reconstruction} and our method. The difference is consistently positive, indicating an improvement for ULTRA-MC.   On the left, the improvement for the lollipop graph
    {\scriptsize \textsf{LOL}$_n$} exceeds $50$, so
    we omit it from the plot.}
    \label{fig:ht-inf-noise}
\end{figure}

\begin{figure}
    \centering
    \includegraphics[width=0.5\columnwidth]{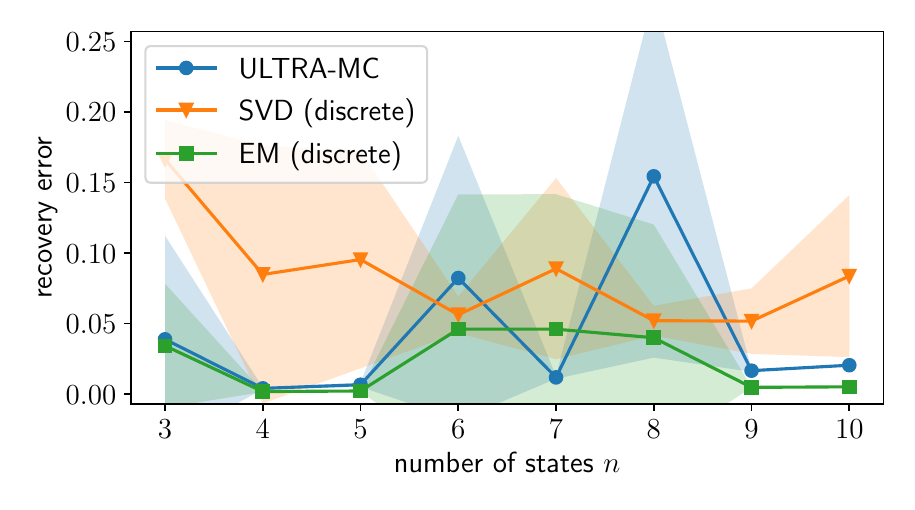}~
    \includegraphics[width=0.5\columnwidth]{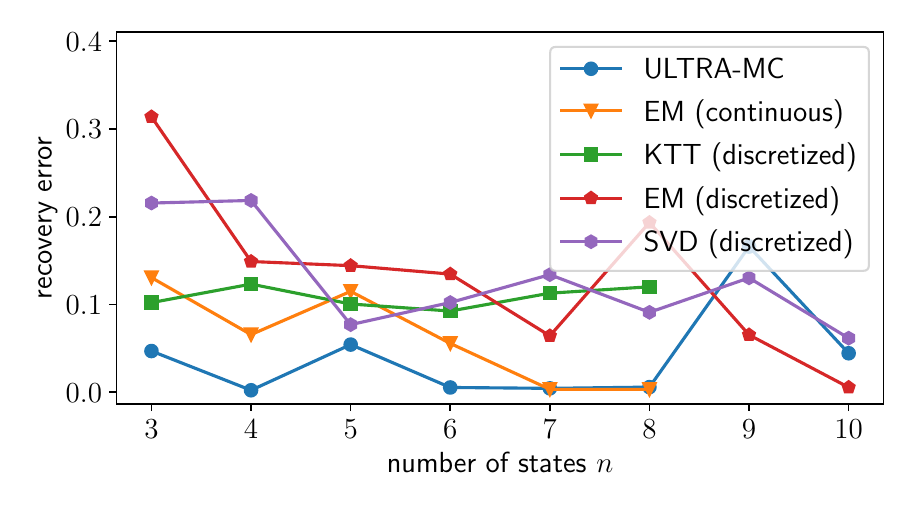}
    \vspace{-2em}
    \caption{Learning a discrete-time (left) and continuous time (right) random mixture of $C=2$ chains from trails. We report average and standard deviation. Missing points indicate a timeout at 2 hours. For readability, we report the standard deviation separately in Table~\ref{tab:std-ht-mix} of Appendix~\ref{sec:apx-exp}}
    \label{fig:ht-mix}
\end{figure}

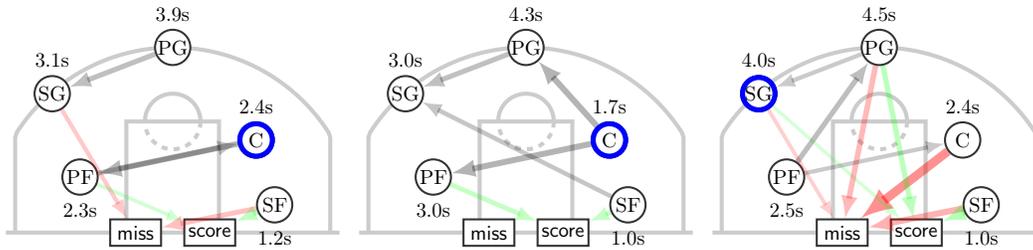
\begin{figure}
    \centering

\tikzstyle{court}=[black!20, line width=2pt,samples=100]
\tikzstyle{player}=[circle, inner sep=0pt, text width=16pt, align=center, draw=black!80, line width=1pt, fill=white]
\tikzstyle{pass}=[-latex, black]
\tikzstyle{basket}=[player, rectangle, inner sep=2pt, minimum height=13pt, text centered, text width=21pt]
\tikzstyle{start}=[player,fill=none,draw=blue,line width=2.5pt]

\scalebox{0.77}{
\begin{tikzpicture}[scale=1.6]

    \draw[court, domain=0:180] plot ({min(1.7, max(-1.7, 2*cos(\x)))}, {2*sin(\x)});
    \draw[court] (-0.5,0) -- (-0.5,1.2) -- (0.5,1.2) -- (0.5,0);
    \draw[court, domain=0:180] plot ({0.3*cos(\x)}, {1.2 + 0.3*sin(\x)});
    \draw[court, dashed, domain=0:180] plot ({0.3*cos(\x)}, {1.2 - 0.3*sin(\x)});
    \draw[court] (-1.8,0) -- (1.8,0);

    \node[player] (PG) at (0,2) {PG};
    \node[player] (SG) at (-1.3,1.5) {SG};
    \node[player] (PF) at (-1.0,0.6) {PF};
    \node[player] (C) at (0.9,1.0) {C};
    \node[player] (SF) at (1.1,0.3) {SF};

    \node[basket] (miss) at (-0.4,0) {$\mathsf{miss}$};
    \node[basket] (score) at (0.4,0) {$\mathsf{score}$};

    \node[label={[label distance=6pt]above:{2.4s}}] at (C) {};
    \node[start] at (C) {};
    \draw[pass,opacity=0.28776581810659885,line width=2.8776581810659883pt] (C) to (PF);
    \node[label={[label distance=6pt]above:{3.1s}}] at (SG) {};
    \draw[pass,opacity=0.21285917911926375,line width=2.1285917911926373pt,red] (SG) to (miss);
    \node[label={[label distance=6pt]above:{3.9s}}] at (PG) {};
    \draw[pass,opacity=0.25427920751148286,line width=2.5427920751148285pt] (PG) to (SG);
    \node[label={[label distance=6pt]below:{2.3s}}] at (PF) {};
    \draw[pass,opacity=0.18244518916453703,line width=1.8244518916453702pt,green] (PF) to (score);
    \draw[pass,opacity=0.23548126658403767,line width=2.3548126658403765pt] (PF) to (C);
    \node[label={[label distance=6pt]below:{1.2s}}] at (SF) {};
    \draw[pass,opacity=0.24770406016131077,line width=2.477040601613108pt,red] (SF) to (miss);
    \draw[pass,opacity=0.24724753796959578,line width=2.472475379695958pt,green] (SF) to (score);

\end{tikzpicture}}
\scalebox{0.77}{
\begin{tikzpicture}[scale=1.6]

    \draw[court, domain=0:180] plot ({min(1.7, max(-1.7, 2*cos(\x)))}, {2*sin(\x)});
    \draw[court] (-0.5,0) -- (-0.5,1.2) -- (0.5,1.2) -- (0.5,0);
    \draw[court, domain=0:180] plot ({0.3*cos(\x)}, {1.2 + 0.3*sin(\x)});
    \draw[court, dashed, domain=0:180] plot ({0.3*cos(\x)}, {1.2 - 0.3*sin(\x)});
    \draw[court] (-1.8,0) -- (1.8,0);

    \node[player] (PG) at (0,2) {PG};
    \node[player] (SG) at (-1.3,1.5) {SG};
    \node[player] (PF) at (-1.0,0.6) {PF};
    \node[player] (C) at (0.9,1.0) {C};
    \node[player] (SF) at (1.1,0.3) {SF};

    \node[basket] (miss) at (-0.4,0) {$\mathsf{miss}$};
    \node[basket] (score) at (0.4,0) {$\mathsf{score}$};

    \node[label={[label distance=6pt]above:{1.7s}}] at (C) {};
    \node[start] at (C) {};
    \draw[pass,opacity=0.29487541693535085,line width=2.9487541693535086pt] (C) to (PG);
    \draw[pass,opacity=0.2949478438375187,line width=2.949478438375187pt] (C) to (PF);
    \node[label={[label distance=6pt]above:{3.0s}}] at (SG) {};
    \node[label={[label distance=6pt]above:{4.3s}}] at (PG) {};
    \draw[pass,opacity=0.250650591089934,line width=2.5065059108993397pt] (PG) to (SG);
    \node[label={[label distance=6pt]below:{3.0s}}] at (PF) {};
    \draw[pass,opacity=0.23231171750782495,line width=2.3231171750782496pt,green] (PF) to (score);
    \node[label={[label distance=6pt]below:{1.0s}}] at (SF) {};
    \draw[pass,opacity=0.22755443550511134,line width=2.2755443550511134pt,green] (SF) to (score);
    \draw[pass,opacity=0.22521480260172205,line width=2.2521480260172204pt] (SF) to (SG);

\end{tikzpicture}}
\scalebox{0.77}{
\begin{tikzpicture}[scale=1.6]

    \draw[court, domain=0:180] plot ({min(1.7, max(-1.7, 2*cos(\x)))}, {2*sin(\x)});
    \draw[court] (-0.5,0) -- (-0.5,1.2) -- (0.5,1.2) -- (0.5,0);
    \draw[court, domain=0:180] plot ({0.3*cos(\x)}, {1.2 + 0.3*sin(\x)});
    \draw[court, dashed, domain=0:180] plot ({0.3*cos(\x)}, {1.2 - 0.3*sin(\x)});
    \draw[court] (-1.8,0) -- (1.8,0);

    \node[player] (PG) at (0,2) {PG};
    \node[player] (SG) at (-1.3,1.5) {SG};
    \node[player] (PF) at (-1.0,0.6) {PF};
    \node[player] (C) at (0.9,1.0) {C};
    \node[player] (SF) at (1.1,0.3) {SF};

    \node[basket] (miss) at (-0.4,0) {$\mathsf{miss}$};
    \node[basket] (score) at (0.4,0) {$\mathsf{score}$};

    \node[label={[label distance=6pt]above:{2.4s}}] at (C) {};
    \draw[pass,opacity=0.4355642865772146,line width=4.355642865772146pt,red] (C) to (miss);
    \node[label={[label distance=6pt]above:{4.0s}}] at (SG) {};
    \node[start] at (SG) {};
    \draw[pass,opacity=0.18847035461070993,line width=1.8847035461070993pt,red] (SG) to (miss);
    \draw[pass,opacity=0.16649656039718758,line width=1.6649656039718757pt,green] (SG) to (score);
    \node[label={[label distance=6pt]above:{4.5s}}] at (PG) {};
    \draw[pass,opacity=0.275283053926425,line width=2.75283053926425pt,red] (PG) to (miss);
    \draw[pass,opacity=0.2655546602757131,line width=2.655546602757131pt,green] (PG) to (score);
    \draw[pass,opacity=0.21152264889210112,line width=2.115226488921011pt] (PG) to (SG);
    \node[label={[label distance=6pt]below:{2.5s}}] at (PF) {};
    \draw[pass,opacity=0.2094435744217447,line width=2.094435744217447pt] (PF) to (C);
    \draw[pass,opacity=0.2561104659893206,line width=2.5611046598932057pt] (PF) to (PG);
    \node[label={[label distance=6pt]below:{1.0s}}] at (SF) {};
    \draw[pass,opacity=0.30516502468504336,line width=3.051650246850434pt,red] (SF) to (miss);
    \draw[pass,opacity=0.28519177390091816,line width=2.8519177390091817pt,green] (SF) to (score);

\end{tikzpicture}}
\caption{Three out of $C=6$ strategies of the Denver Nuggets. The remaining
ones are in Figure~\ref{fig:nba}.
We mark the six positions in a basketball game: Point Guard (PG), Shooting Guard (SG), Power Forward (PF), Center (C), and Small Forward (SF). Each position is annoted with the average ball holding time. Arrow thickness and opacity reflect the probability of a pass, and
we omit passes that occur with probability less than 0.2
for clarity.
The player most likely to start a passing game is highlighted in blue. Attempted shots are indicated in red (miss) and green (score).}
\label{fig:nuggets}
\end{figure}

\spara{Estimating Hitting Times}
To establish a baseline for our method, we need
to understand how accurately we can deduce hitting
times from a specific number of trails of certain length.
We use the four graph types $K_n$, $S_n$,
{\small \textsf{LOL}$_n$}, and $G_{\sqrt{n}\times \sqrt{n}}$ to
assess the basic estimation of hitting times
(cf. Algorithm~\ref{alg:ht-est}).
In Figure~\ref{fig:ht-est}, we
vary the number of trails and their lengths and quantify
the error between the true hitting times
matrix $\trueH$ and estimations $\estH$
via the Frobenius norm.
We find that $K_n$ is easiest in terms of sampling,
as all hitting times are identical and linear in $n$.
In contrast, the star graph $S_n$ demands longer trails
to achieve comparable accuracy, owing to the asymmetry
in its hitting times. A grid displays similar behavior.
Conversely, for the lollipop graph, much longer trails
are necessary to observe an acceptable estimation error.
Overall, we observe that the estimation accuracy
depends on the underlying Markov chain and the
magnitude and skew of the hitting times. The
number of sampled trails and their length needs to be
chosen accordingly.

\spara{Learning a Single Markov Chain}
To understand how robust our learning
algorithm is under noise, especially compared
to the method of \citet{wittmann2009reconstruction},
we apply both methods to noisy hitting times.
We either estimating the hitting times
from trails or add homoscedastic Gaussian noise.
Specifically,
we set $\estH_{uv} = \trueH_{uv} + N(0, \sigma^2)$
for each $u \not= v$ independently.
In Figure~\ref{fig:ht-inf-noise},
we report the improvement in recovery error
over the method
of \citet{wittmann2009reconstruction},
as a function of the length of the sampled
trails and the standard deviation $\sigma$.
We consistently improve,
particularly when
the noise stems from estimating
hitting times from trails.  
Figures~\ref{fig:ht-inf-from-trails} and
\ref{fig:ht-inf-noise-hetero} in
Appendix~\ref{sec:apx-exp} show the
recovery error for both approaches, and
also with heteroscedastic noise.

We also study the scalability of our method.
Table~\ref{tab:runtime} shows
execution times for various gradient implementations,
applied to a single random discrete-time Markov chain.
We compare analytical gradients utilizing the formulae detailed in Section~\ref{sec:proposed}, and the methods of automatic differentiation (Autodiff) and numerical analysis (forward difference method).
We see that our exact analytical derivatives are crucial
for scalability as the other two methods are not even scalable
to $\approx 50$ nodes.
Furthermore, Figure~\ref{fig:ht-inf} of Appendix~\ref{sec:apx-exp}
shows the convergence behavior.
On the left, we display recovery error
and loss for a fixed number of $10\,000$ gradient
descent iterations using the true hitting
times $\trueH$.
We see that $10\,000$ iterations suffice
for a relatively large number of states $n$, but more
iterations are necessary when $n > 500$.
On the right, we show
the convergence of our method from a random
initialization. We see that within only a
few thousand iterations, we reach and
then improve the recovery error obtained by
\citet{wittmann2009reconstruction}.

\spara{Learning Mixtures}
In Figure~\ref{fig:ht-mix},
we conduct experiments using $C=2$ chains for the discrete and continuous-time settings. In discrete-time, we learn with 1000 trails
and for the more challenging continuous-time setting, we increase this to 5000 trails.
Trails are of length 1000 in both scenarios. Notably, in the continuous-time setting, {\small \textsf{EM (continuous)}} and {\small \textsf{KTT (discretized)}} exceeded a duration of 2 hours, leading to a timeout.
Our findings indicate that \ULTRAMC is able to match the performance in discrete-time and surpasses the performance in continuous-time of other competing approaches, demonstrating improved effectiveness and scalability.
We evaluate our method on \emph{Markovletics} \cite{spaeh2024www} which unmixes offensive strategies from the passing game in the NBA.
We showcase six strategies of the Denver Nuggets
in Figures~\ref{fig:nba} and \ref{fig:nuggets}.
We also include a mixture of offensive strategies of the Boston Celtics in Appendix~\ref{sec:apx-exp}.
Although the assumption that the passing game follows a stochastic Markov process is simplistic, it still gives us direct insights about the passing game.


\begin{figure}
    \centering
    \includegraphics[trim=20mm 0mm 20mm 0mm,clip,width=0.4\textwidth]{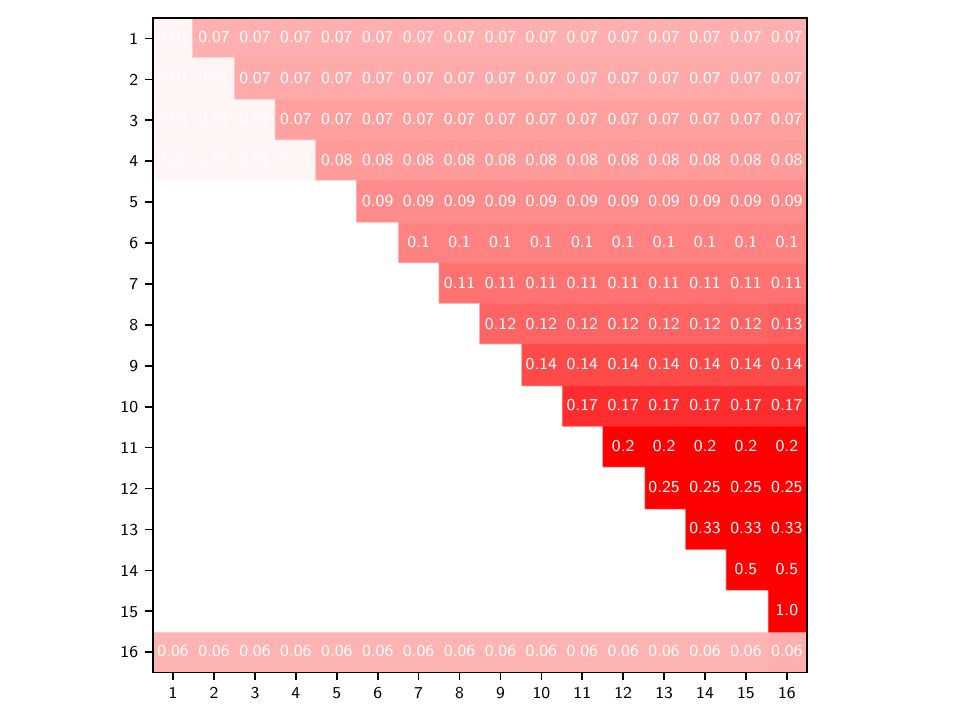}
    \hspace{1em}
    \includegraphics[trim=20mm 0mm 20mm 0mm,clip,width=0.4\textwidth]{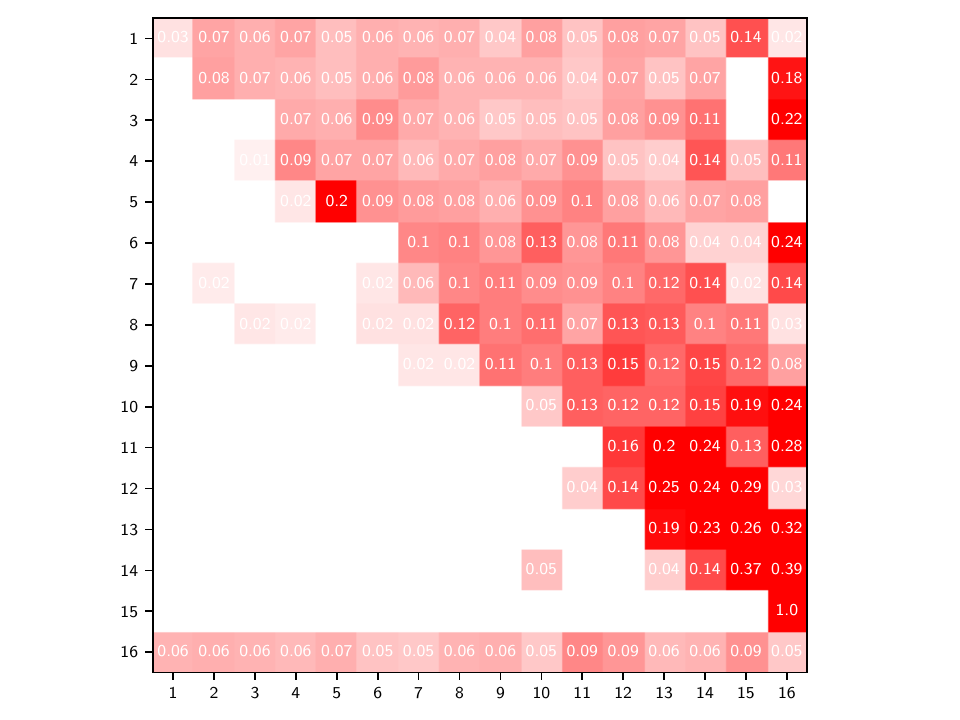}
    \caption{Learning a Markov chain on a complete DAG on $n=16$ nodes. The tables represent the
        transition matrices learned from the true hitting times (left) and noisy hitting
        times with $\sigma=0.3$ (right). Note that the true Markov chain has non-zero transition probabilities
        only for all $1 \le u < v \le n$ where $M_{uv} = \frac 1 {n-u}$ and $M_{16,16} = 1$.}
    \label{fig:dag}
\end{figure}

Figure~\ref{fig:boston} shows the $C=6$ learned offensive strategies from Boston Celtics' passing game during the 2022 season. The format of the plots shows the same format explain in Section~\ref{sec:exp}.

\spara{Violating Irreducibility}
In theory, our method requires aperiodic and irreducible chains in order to obtain Equation~\ref{eq:1}. Yet, Markov chains observed in practice may violate these conditions, such as cancer progression networks~\cite{beerenwinkel2005mtreemix,tsourakakis2013modeling,desper1999inferring}.
To confirm that our approach is still practically applicable to
such instances,
we create trails from a directed acyclic graph (DAG) by adding every edge $(u,v)$ for $1 \le u<v\leq n$. For all pairs $u,v$, we calculate the estimated hitting times and assign a high value to the hitting time $H(u,v)$ whenever $u>v$. By setting the infinite hitting times to a value substantially greater than any observed hitting times, we effectively render the chain irreducible. This approach allows us to confirm that our algorithm can accurately learn the DAG structure by merely filtering out transitions with low probabilities.  

We learn a Markov chain by replacing infinite hitting
times with the value $100$, which is sufficiently large.
As this introduces new, previously impossible
transitions into the chain, we only retain transitions with sufficiently
large mass.
Specifically, if the transition probability from state $u$ to $v$
has probability $M_{uv} \le 0.1 \cdot \max_{w} M_{uw}$, we remove it from
the chain.
The results in Figure~\ref{fig:dag} 
indicate that in practical scenarios, we can circumvent the limitations of using our tools on irreducible chains by treating them as if they were irreducible, through assigning high values to the absent hitting times. A thorough investigation of this strategy merits its own dedicated study, which we aim to pursue as part of future research into understanding cancer progression through the use of mixed Markov chains~\cite{desper1999inferring,beerenwinkel2005mtreemix}.

 \vspace{-3mm}
\section{Conclusion}
In this work, we introduce \ULTRAMC, a pioneering algorithm designed to efficiently learn either a single Markov chain or a mixture of Markov chains from a subset of potentially noisy hitting times. This is the first algorithm of its kind. Compelling questions to explore involve the development of more efficient and precise algorithms for learning mixtures.  the design of improved estimators for deducing hitting times from trails, and better selection of the number of chains $C$.

\bibliography{ref}

\newpage
\appendix

\section{Omitted Algorithms and Proofs}
\label{sec:apx-omitted}

\subsection{Estimating Hitting Times}

In Algorithm~\ref{alg:ht-est},
we describe the estimation of
hitting times from a set of
trails $\mathbf X$.
We state the algorithm
in the discrete-time setting,
but it naturally extends to   the 
continuous-time setting as well. 
The algorithm is
an efficient
implementation of the following idea:
For each trail
$\mathbf x \in \mathbf X$,
each index $t \in \N$ with state $\mathbf x_t = u$,
and each other state $v \not= u$,
we record the time
$\min \{ t' \ge t : \mathbf x_{t'} = v \}$
as a sample for the hitting time
from $u$ to $v$. We then output
the average over all these samples.

It is important to acknowledge that the method for estimating hitting times as outlined in Algorithm~\ref{alg:ht-est} carries a bias. This can be illustrated with a simple Markov chain example, containing two nodes, $u$ and $v$, where $v$ can be reached from $u$ through either a short or a long path. In scenarios where only the shorter path's trajectories are observed due to the limited length of trails, and none from the longer path, the estimation might skew towards the shorter paths. In the case of a singular chain, this issue is manageable by linking paths together to simulate the longer journey, thus allowing for an unbiased estimation of hitting times. However, this strategy falls short in scenarios involving multiple chains because it relies on potentially inaccurate clustering, preventing seamless path concatenation without introducing further errors, as the original chain from which a trail is derived remains unknown. Consequently, in the context of estimating hitting times using Algorithm~\ref{alg:ht-est}—especially when trails do not adequately cover long paths—the estimates are accepted to be biased. This method, however, avoids the pitfalls of adding errors through trail concatenation and still provides a reliable estimate of hitting times under these constraints. 

In real-world scenarios, it might not always be possible to observe a transition from state $u$ to $v$. How we choose to handle the corresponding estimate $\hat{H}(u,v)$ can vary depending on the situation. In many of our experiments, we treat it as missing data, opting not to include it in our fitting process. In cases where we determine that a transition between these states is impossible, we assign it a significantly high value (e.g., Q6 in Section~\ref{sec:exp}). For different applications, especially those where the observation period is short relative to the time needed for a transition, the development of right-censoring methods is advisable. Interestingly, the challenge of deriving hitting times from trails has yet to be thoroughly explored. The techniques described in \cite{cherapanamjeri2019testing,daskalakis2018testing} may be beneficial in achieving this objective. It should be highlighted that the complexity of the issue decreases substantially when the trails are sufficiently lengthy, specifically of the order $\Omega(n^3)$. This simplification can be demonstrated by employing concentration methods akin to those found in the works of \cite{benjamini2006waiting,spaeh2024www}. It should be noted that although this issue is a classic one, the problem of deducing hitting times from trails of a specific length remains under-explored.

\begin{algorithm}[tb]
   \caption{Estimating Hitting Times in discrete-time}
   \label{alg:ht-est}
\begin{algorithmic}
   \STATE {\bfseries Input:} Set of trails $\mathbf X$
   \STATE {\bfseries Output:} Hitting times $\estH$
   \STATE $T_{uv} \gets 0$
   \STATE $C_{uv} \gets 0$
   \FOR{each trail $\mathbf x \in \mathbf X$}
       \STATE Initialize empty sequence $\texttt P_u$ for each $u \in V$
       \FOR{$t = 1$ to $\mathrm{length}(\mathbf x)$}
           \STATE $u \gets x_t$
           \STATE Append $t$ to $\texttt P_u$
       \ENDFOR
       \STATE Let $i_u \gets 1$ for each $u \in V$
       \FOR{$t = 1$ to $\text{length}(\mathbf x)$}
           \STATE $u \gets x_t$
           \FOR{$v \in V$}
                \IF{$i_v \le \mathrm{length}(\texttt P_v)$}
                   \STATE $T_{uv} \gets T_{uv} + \texttt P_v(i_v) - t$
                   \STATE $C_{uv} \gets C_{uv} + 1$
                \ENDIF
           \ENDFOR
       \ENDFOR
   \ENDFOR
   \STATE $\estH_{uv} \gets T_{uv} / C_{uv}$
\end{algorithmic}
\end{algorithm}

\subsection{Omitted Proofs}

In this section, we give
proofs and lemmas
that we omitted in the main body.
We start by showing the
expression for the hitting times
for discrete-time and continuous-time
Markov chains in Section~\ref{apx:ht}.
Next, we show how to obtain
Lemmas~\ref{lem:gradA} and
\ref{lem:gradB} in Section~\ref{apx:grad}.
The following Sections
\ref{apx:stationary} and
\ref{apx:lap}
contains Lemmas detailing
the Jacobian of the stationary
distribution and the Laplacian,
respectively.

\subsubsection{Hitting Times}
\label{apx:ht}

\begin{lemma}
\label{lem:ht-discrete}
The hitting time from state $u$ to state $v$
for a discrete-time Markov chain with
Laplacian $L$ is given by
\[
    H_{uv} = \left(\ones - \frac 1 {s_v}\ones_v\right)^\top
        L^+ \left(\ones_u - \ones_v\right) .
\]
\end{lemma}

For completeness, we repeat the proof due to
\cite{cohen2016}.

\begin{proof}
For two states $u, v \in [n]$ we either have
$H_{uv} = 0$ if $u = v$ or, by the law of total
expectation,
\begin{align}
    \label{eq:5}
    H_{uv} = \sum_{w} (1 + H_{wv}) M_{uw}
    = 1 + (M H)_{uv} .
\end{align}
We now fix $v$ and let $I = V \setminus \{v\}$.
Then, \eqref{eq:5} over all states $u$ is equivalent
to
\[
\begin{pmatrix}L_{I,I} & L_{I,v}\\
0 & 1
\end{pmatrix}\cdot\begin{pmatrix}H_{I,v}\\
H_{v,v}
\end{pmatrix}=\begin{pmatrix}1\\
0
\end{pmatrix} .
\]
We can construct a solution to this system:
Let $d=\ones-\frac{1}{s_{v}}\ones_{v}$
such that $d^{\top}s=0$
Let $x$ be any solution to $L x=d$ and
note that $x$ is unique up to adding multiples of $\ones$.
Therefore,
$y=x-x_{v} \ones$ is a solution with
$Ly=d$ and
$y_{v}=\zeros$ and $y$ therefore
satisfies the above linear system. Due
to uniqueness, $H_{u,v}=y_u$ and
\[
H_{uv}=y_{u}=\ones_{u}^{\top}\underbrace{L^{+}d}_{=x}-\underbrace{\ones_{v}^{\top}L^{+}d}_{=x_{v}}=\left(\ones_{u}-\ones_{v}\right)^{\top}L^{+}\underbrace{\left(\ones-\frac{1}{s_{v}}\ones\right)}_{=d}.
\]
\end{proof}

There is no analog expression for CTMCs. We state such an extension as the following lemma. 

\begin{lemma}
\label{lem:ht-continuous}
The hitting time from state $u$ to state $v$
for a continuous-time Markov chain with
rate matrix $K$ is given by
\[
    H_{uv} = \left(\ones - \frac 1 {s_v}\ones_v\right)^\top
        (-K)^+ \left(\ones_u - \ones_v\right) .
\]
\end{lemma}

\begin{proof}
    We derive the result analogously to
    Lemma~\ref{lem:ht-discrete}.
    Recall that in a continuous-time Markov
    chain the transition probabilities
    are given as the matrix exponential
    $M(\tau) = e^{\tau K}$.
    Let now $H(\tau)$ be the hitting
    times in $M(\tau)$ and note that
    $H = \lim_{\tau \to 0} H(\tau)$.
    In the chain $M(\tau)$, we
    again obtain through
    the law of total expectation that
    \[
        H(\tau)_{uv} = 
        \tau+\sum_{w}H(\tau)_{w,v} M(\tau)
    \]
    for states $u\not=v$.
    Note that this is just a scaled version of
    \eqref{eq:5} in the proof of
    Lemma~\ref{lem:ht-discrete},
    The solution to this system is therefore
    a scaled version of the solution to \eqref{eq:5},
    i.e.
    \[
        H(\tau)_{uv} = \tau \left(\ones - \frac 1 {s_v}\ones_v\right)^\top
            L(\tau)^+ \left(\ones_u - \ones_v\right)
    \]
    where $L(\tau) = I - M(\tau)$.
    Finally, 
    \begin{align*}
        H_{uv}
        &= \lim_{\tau \to 0} H(\tau)_{uv} \\
        &= \lim_{\tau \to 0} \tau \left(\ones - \frac 1 {s_v}\ones_v\right)^\top
            L(\tau)^+ \left(\ones_u - \ones_v\right) \\
        &= \left(\ones - \frac 1 {s_v}\ones_v\right)^\top
            \left(\lim_{\tau \to 0} \frac 1 \tau L(\tau)\right)^+ \left(\ones_u - \ones_v\right) \\
        &= \left(\ones - \frac 1 {s_v}\ones_v\right)^\top
            (-K)^+ \left(\ones_u - \ones_v\right)
    \end{align*}
    as $\frac 1 \tau L(\tau) = 
    \frac 1 \tau (I - e^{\tau K}) \to -K$
    for $\tau \to 0$.
\end{proof}

\subsubsection{Computing the Gradient of the Loss}
\label{apx:grad}

We now complete the derivation of the
efficient gradient expression
from Section~\ref{sec:prop-single}.
We begin by proving Lemma~\ref{lem:stationary}
which we need to compute the stationary
distribution and its derivative.

{
\textbf{Lemma \ref{lem:split}.} \it
Let
$a=\ones^{\top}L^{+}\in\R^{1\times n}$
and set
\begin{align*}
    A &= a^{\top}\ones^{\top}-\ones a \\
    B &= \left(L^{+}-\diag\left(L^{+}\right)\ones^{\top}\right)^{\top}\diag\left(s\right)^{-1} .
\end{align*}
Then, $H = A - B$.
}
\begin{proof}
Pick a pair of states $u, v \in [n]$.
We can re-write Equation \eqref{eq:1} as
\begin{equation}
\label{eq:2}
H_{uv}=\ones^{\top}L^{+}\chi_{vu}-\frac{1}{s_v}\uvec_v^{\top}L^{+}\chi_{vu}.
\end{equation}
The first term evaluates to $\ones^{\top}L^{+}\chi_{vu}=a_u-a_v = A_{uv}.$
The second term is
\begin{align*}
\frac{1}{s_v}\uvec_v^{\top}L^{+}\chi_{vu}
&=\frac{1}{s_v}\left(L_{vu}^{+}-L_{vv}^{+}\right) \\
&=\frac{1}{s_v}\left(L^{+}-\diag\left(L^{+}\right)\ones^{\top}\right)_{vu} \\
&=\left(\left(L^{+}-\diag\left(L^{+}\right)\ones^{\top}\right)^{\top}\diag\left(s\right)^{-1}\right)_{uv}
=B_{uv}
\end{align*}
\end{proof}

{
\textbf{Lemma \ref{lem:stationary}.} \it
The stationary distribution for a Markov chain with
Laplacian $L$ is $s=\frac{d}{\| d \|_1}$ where $d=\ones - L L^+ \ones$.
}

\begin{proof}
The stationary distribution is the probability vector $s\in\Delta_{n}$
such that $s_u=\sum_v M_{vu} s_v$ for all $u\in[n]$. The latter
is equivalent to $s^\top = s^\top M$ or $s^\top L= \zeros$.
Note that $L L^+$ is symmetric,
and therefore
\[
    d^\top L = \ones^\top L - \ones^\top (L L^+)^\top L = \ones^\top L - \ones^\top L L^+ L .
\]
By properties of the pseudoinverse,
$L L^+ L = L$ and thus $d^\top L = \zeros$ which
implies $s^\top L = \frac 1 {\| d \|_1} d^\top L = \zeros$.
It remains to show that $s \in \Delta_n$,
for which we need that $d \ge \zeros$
(coordinate-wise).
By definition of $d$, the latter is equivalent
to $\ones \ge L L^+ \ones$.
It is well known that $L L^+$
is an orthogonal projection and
thus has eigenvalues either
$0$ or $1$ and
therefore,
all entries of $L L^+ \ones$
have value at most $1$.
\end{proof}

Recall that we split the
computation of the gradient
into two terms
\[
    \nabla \ell_2(L^+) =
    \left( \frac{\partial A}{\partial L^+} \right)^\top \otimes \Delta +
    \left( \frac{\partial B}{\partial L^+} \right)^\top \otimes \Delta
\]
We now prove Lemmas~\ref{lem:gradA}
and \ref{lem:gradB} which give
expressions for both terms.

{
\textbf{Lemma \ref{lem:gradA}.} \it
We have
\[
    \left( \frac{\partial A}{\partial L^+} \right)^\top \otimes \Delta
    = \ones\left(\ones^{\top}\Delta\right)-\left(\ones^{\top}\Delta\right)^{\top}\ones^{\top}-\ones\left(\Delta\ones\right)^{\top}+\left(\Delta\ones\right)\ones^{\top} .
\]
Furthermore,
$\left( \frac{\partial A}{\partial L^+} \right)^\top \otimes \Delta$
can be computed in time $O(n^2)$.
}

\begin{proof}
Note that
\[
\frac{\partial A}{\partial L_{uv}^{+}} =\frac{\partial}{\partial L_{uv}^{+}}\left(a^{\top}\ones^{\top}-\ones a\right)
  =\frac{\partial}{\partial L_{uv}^{+}}\left(L^{\top+}\ones\ones^{\top}-\ones\ones^{\top}L^{+}\right)
  =\frac{\partial}{\partial L_{uv}^{+}}\left(\left(\ones\ones^{\top}L^{+}\right)^{\top}-\ones\ones^{\top}L^{+}\right).
\]
Recall that
$\frac{\partial L^+}{\partial L^+_{uv}} = \ones \chi_{vu}^\top$
which implies
\[
\frac{\partial}{\partial L_{uv}^{+}}\ones\ones^{\top}L^{+} 
 =\ones\chi_{vu}^{\top}
\]
and thus
\[
\frac{\partial A}{\partial L_{uv}^{+}}=\chi_{vu}\ones^{\top}-\ones\chi_{vu}^{\top} .
\]
Finally,
\begin{multline*}
\ones^{\top}\left(\left(\frac{\partial A}{\partial L_{uv}^{+}}\right)\circ\Delta\right)\ones
 =\ones^{\top}\left(\ones\left(\uvec_v^{\top}-\uvec_u^{\top}\right)\circ\Delta-\chi_{vu}\ones^{\top}\circ\Delta\right)\ones\\
  =\ones^{\top}\left(\ones\left(\uvec_v^{\top}-\uvec_u^{\top}\right)\circ\Delta\right)\ones-\ones^{\top}\left(\chi_{vu}\ones^{\top}\circ\Delta\right)\ones
  =\left(\ones^{\top}\Delta\right)_v-\left(\ones^{\top}\Delta\right)_u-\left(\Delta\ones\right)_v+\left(\Delta\ones\right)_u
\end{multline*}
and thus
\begin{multline*}
\left(\frac{\partial A}{\partial L^{+}}\right)^{\top} \otimes \Delta
 =\left(\left(\ones^{\top}\Delta\right)_v-\left(\ones^{\top}\Delta\right)_u-\left(\Delta\ones\right)_v+\left(\Delta\ones\right)_u\right)_{uv} \\
  =\ones\left(\ones^{\top}\Delta\right)-\left(\ones^{\top}\Delta\right)^{\top}\ones^{\top}-\ones\left(\Delta\ones\right)^{\top}+\left(\Delta\ones\right)\ones^{\top} .
\end{multline*}
Note that this term only involves
summing over rows and columns, which
can be done in $O(n^2)$.
\end{proof}

{
\textbf{Lemma \ref{lem:gradB}.} \it
We have
\[
    \left( \frac{\partial B}{\partial L^+} \right)^\top \otimes \Delta
    = \diag(s)^{-1} \ones \ones^\top (\Delta^\top - \Delta)
     + \Delta \diag(s)^{-1} \ones \ones^\top + G
\]
for
\[
    G = \left(\ones^{\top}\left(\left(L^{+}-\diag\left(L^{+}\right)\ones^{\top}\right)^{\top}D\circ\Delta\right)g_{uv}\right)_{uv} .
\]
Furthermore, 
$\left( \frac{\partial B}{\partial L^+} \right)^\top \otimes \Delta$ can be computed in time
$O(n^3)$.
}

\begin{proof}
We pick a pair of states $u \not= v$.
By definition of $B$ and the product rule,
\begin{align}
\frac{\partial B}{\partial L_{uv}^{+}} & =\frac{\partial}{\partial L_{uv}^{+}}\left(L^{+}-\diag\left(L^{+}\right)\ones^{\top}\right)^{\top}\diag\left(s\right)^{-1} \nonumber \\
 & =\left(\frac{\partial}{\partial L_{uv}^{+}}\left(L^{+}-\diag\left(L^{+}\right)\ones^{\top}\right)^{\top}\right)\diag\left(s\right)^{-1} \label{eq:fst} \\
 &\quad+\left(L^{+}-\diag\left(L^{+}\right)\ones^{\top}\right)^{\top}\left(\frac{\partial}{\partial L_{uv}^{+}}\diag\left(s\right)^{-1}\right) \label{eq:snd} .
\end{align}
We consider the two terms \eqref{eq:fst} and \eqref{eq:snd} separately.
To obtain \eqref{eq:fst}, we calculate
\[
    \frac{\partial}{\partial L_{uv}^{+}}\left(L^{+}-\diag\left(L^{+}\right)\ones^{\top}\right)
    =\ones\chi_{vu}^{\top}+\uvec_u\ones^{\top}
\]
and thus have that
\[
    \eqref{eq:fst} = \left( \ones\chi_{vu}^{\top}+\uvec_u\ones^{\top} \right) \diag(s)^{-1} .
\]
For \eqref{eq:snd}, we use Lemma~\ref{lem:stat-jac} and obtain
\begin{align*}
\frac{\partial}{\partial L_{uv}^{+}}\diag\left(s\right)^{-1} 
 & =\diag\left(\left(\frac{1}{d_{w}}\left(\frac{\|d\|_{1}}{d_{w}}\uvec_{w}^{\top}-\ones^{\top}\right)g_{uv}\right)_{w}\right)\\
 & =\diag\left(g_{uv}\right) D .
\end{align*}
for $D = \diag\left(\left(\frac{1}{d_{w}}\left(\frac{1}{s_{w}}\uvec_{w}^{\top}-\ones^{\top}\right)\right)_{w}\right)$.
Thus,
\[
    \eqref{eq:snd} = 
 \left(L^{+}-\diag\left(L^{+}\right)\ones^{\top}\right)^{\top} \diag(g_{uv}) D.
\]
and overall, for $s^{-1} = (s^{-1}_w)_w$,
\begin{align*}
    \ones^\top \left( \frac{\partial B}{\partial L^+_{uv}} \circ \Delta \right) \ones
    &= \ones^\top \left( \left(\ones\chi_{vu}^{\top}+\uvec_u\ones^{\top}\right)\diag\left(s\right)^{-1}\circ\Delta \right) \ones \\
    &\quad + \ones^\top \left( \left(L^{+}-\diag\left(L^{+}\right)\ones^{\top}\right)^{\top} \diag(g_{uv}) D \circ \Delta \right) \ones \\
    &= \ones^{\top}\left(\ones\chi_{vu}^{\top}\circ\Delta+\uvec_u\ones^{\top}\circ\Delta\right)s^{-1} \\
    &\quad + \ones^{\top}\left(\left(L^{+}-\diag\left(L^{+}\right)\ones^{\top}\right)^{\top}D\circ\Delta\right)g_{uv} \\
    &= \frac{1}{s_v}\left(\ones^{\top}\Delta\right)_v-\frac{1}{s_u}\left(\ones^{\top}\Delta\right)_u+\uvec_u^{\top}\Delta s^{-1} \\
    &\quad + \ones^{\top}\left(\left(L^{+}-\diag\left(L^{+}\right)\ones^{\top}\right)^{\top}D\circ\Delta\right)g_{uv}
\end{align*}
where the second equality holds because $\diag(g_{uv})$ and $\diag(s)$ are diagonal.
Therefore,
\begin{align*}
    \left( \frac{\partial B}{\partial L^+_{uv}} \right)^\top \otimes \Delta &=
    \left( \diag(s)^{-1} \ones \ones^\top \Delta \right)^\top - \diag(s)^{-1} \ones \ones^\top \Delta
     + \Delta \diag(s)^{-1} \ones \ones^\top + G \\
    &= \diag(s)^{-1} \ones \ones^\top (\Delta^\top - \Delta)
     + \Delta \diag(s)^{-1} \ones \ones^\top + G .
\end{align*}
Note that we can clearly compute
$\diag(s)^{-1}\ones \ones^\top (\Delta^\top - \Delta)$
and $\Delta \diag(s)^{-1} \ones^\top \ones$
in time $O(n^3)$.
It remains to show that $G$
can be computed in time $O(n^3)$:
Note that we can pre-compute
the vector
\[\ones^{\top}\left(\left(L^{+}-\diag\left(L^{+}\right)\ones^{\top}\right)^{\top}D\circ\Delta\right)\]
in time $O(n^3)$.
Furthermore, since
\[
    g_{uv} = \uvec_u\chi_{vu}^{\top}(LL^{+} + I )\ones,
\]
we can also pre-compute the vector $(LL^{+} + I )\ones$
and then determine $g_{uv}$ in time $O(n)$ for each $u, v \in [n]$.
Finally, we can compute each inner product to obtain
the entry $G_{uv}$ in time $O(n)$, so overall we need
$O(n^3)$ to compute $G$.
\end{proof}

\subsubsection{Jacobian of the Stationary Distribution}
\label{apx:stationary}

\begin{lemma}
    \label{lem:stat-jac}
    For any pair $u \not= v$ and state $w$,
    \[
        \frac{\partial s_w^{-1}}{\partial L_{uv}^+} =
        \frac{1}{d_{w}}\left(\frac 1 {s_w}\uvec_{w}^{\top}-\ones^{\top}\right) L g_{uv}
    \]
    where
    $g_{uv} = \uvec_u\chi_{vu}^{\top}(LL^{+} + I )\ones$.
\end{lemma}

\begin{proof}
Recall that
$s=\frac{d}{\| d \|_1}$ for $d=\ones - L L^+ \ones$.
Let us thus first calculate
$\frac{\partial d}{\partial L^+_{uv}}$.
We have
\begin{multline*}
\frac{\partial d}{\partial L_{uv}^{+}}  =\frac{\partial}{\partial L_{uv}^{+}}\left(\ones-LL^{+}\ones\right)
  =-\frac{\partial}{\partial L_{uv}^{+}}LL^{+}\ones\\
  =-\left(\frac{\partial L}{\partial L_{uv}^{+}}\right)L^{+}\ones-L\left(\frac{\partial L^{+}}{\partial L_{uv}^{+}}\right)\ones
  =-L\uvec_u\chi_{vu}^{\top}LL^{+}\ones-L\uvec_u\chi_{vu}^{\top}\ones
  =-L \underbrace{\uvec_u\chi_{vu}^{\top}(LL^{+} + I )\ones}_{=g_{uv}}
\end{multline*}
where we use Corollary~\ref{cor:lapl-jac} to evaluate
$\left( \frac{\partial L}{\partial L^+_{uv}} \right) L^+ \ones$.
Since $d\ge\zeros$, we furthermore
obtain for the length that
\[
 \frac{\partial}{\partial L_{uv}^{+}}\|d\|_{1}
 =\frac{\partial}{\partial L_{uv}^{+}}\ones^{\top}d
 =\ones^{\top}\frac{\partial d}{\partial L_{uv}^{+}}
 =-\ones^{\top}L g_{uv} .
\]
Given this,
we can now proceed to
calculate the Jacobian
of the stationary distribution.
We use the product rule to decompose
\begin{align*}
 \frac{\partial s_{w}^{-1}}{\partial L_{uv}^{+}}
 =\frac{\partial}{\partial L_{uv}^{+}}\frac{\|d\|_{1}}{d_{w}} 
 =\frac{1}{d_{w}}\left(\frac{\partial}{\partial L_{uv}^{+}}\|d\|_{1}\right)-\frac{\|d\|_{1}}{d_{w}^{2}}\left(\frac{\partial d_{w}}{\partial L_{uv}^{+}}\right) .
\end{align*}
We plug in our previous calculations for $d$
and obtain
\begin{align*}
\frac{\partial s_{w}^{-1}}{\partial L_{uv}^{+}}
 =\left( -\frac{1}{d_{w}} \ones^{\top} + \frac{\|d\|_{1}}{d_{w}^{2}} \uvec_w^\top \right)
 L g_{uv} 
 =\frac{1}{d_{w}}\left(\frac 1 {s_w}\uvec_{w}^{\top}-\ones^{\top}\right) L g_{uv} .
\end{align*}
\end{proof}

\subsubsection{Jacobian of the Laplacian}
\label{apx:lap}

\begin{lemma}
\label{lem:lapl-jac}
The Laplacian has
the following Jacobian
with respect to $L^+_{uv}$, for
any $u \not= v$:
\[
    \frac{\partial L}{\partial L^+_{uv}}
    = LL^{\top}\chi_{vu}\uvec_u^{\top}\left(I-L^{+}L\right)
    - L\uvec_u\chi_{vu}^{\top}L .
\]
\end{lemma}

\begin{proof}
We want to evaluate the matrix
$\frac{\partial L}{\partial L^+_{uv}}$
in closed form.
Let $X = L^+$ and note that
$L = X^+$.
By the definition of the gradient,
\begin{align}
    \label{eq:3}
    \frac{\partial X^+}{\partial X_{uv}} =
    \lim_{\epsilon \to 0} \frac 1 \epsilon
        ((X + \epsilon \uvec_u \chi_{vu}^\top)^+ - X^+) .
\end{align}
We can use a result from \citet{meyer73} to evaluate
the rank-one update
$X + \epsilon \uvec_u \chi_{vu}^\top$ to
the pseudoinverse.
Note that we meet the preconditions of Theorem~5
from \citet{meyer73}:
First, $\chi_{vu}$ is in the
row space of $X = L^+$. This is true
since $L^{+}$ has rank $n-1$, the kernel
of the row space of $L^+$ is spanned
by $\ones^\top$ since $\ones^{\top}L^{+}=\zeros$,
and $\chi_{vu}$ is orthogonal
to $\ones$.
Second, we require that
$\beta=1+\epsilon\chi_{vu} X^+ e_u \not= 0$.
Clearly, $\beta > 0$ for sufficiently small
$\epsilon$.
We can thus apply Theorem~5 from \citet{meyer73} and get 
\[
    (X + \epsilon \uvec_u \chi_{vu}^\top)^+
    = X^{+}+\epsilon\frac{1}{\beta}X^{+}h^{\top}u^{\top}-\epsilon\frac{\beta}{\sigma}pq^{\top}
\]
where
\begin{align*}
\beta & = 1+\epsilon \chi_{vu}^{\top}X^{+} \uvec_u \in \R\\
h & = \chi_{vu}^{\top}X^{+} \in \R^{1 \times n}\\
u & = \left(I-XX^{+}\right) \uvec_u \in \R^n \\
k & = X^{+} \uvec_u \in \R^n \\
p & =-\epsilon\frac{\|u\|^{2}}{\beta}X^{+}h^{\top}- k \in \R^n \\
q^{\top} & =-\epsilon\frac{\|h\|^{2}}{\beta}u^{\top}-h \in \R^{1 \times n}\\
\sigma & =\epsilon^2\|h\|^{2}\cdot\|u\|^{2}+\left|\beta\right|^{2} \in \R
\end{align*}
We plug this back into
\eqref{eq:3} and obtain
\begin{align}
    \label{eq:4}
    \frac{\partial X^+}{\partial X_{uv}}
    = \lim_{\epsilon\to0}\frac{1}{\beta}X^{+}h^{\top}u^{\top}-\lim_{\epsilon\to0}\frac{\beta}{\sigma}pq^{\top}
\end{align}
We treat both limits separately. For the first limit,
we have
\begin{align*}
\lim_{\epsilon\to0}\frac{1}{\beta}X^{+}h^{\top}u^{\top} 
= \left(\lim_{\epsilon\to0}\frac{1}{1+\epsilon \chi_{vu}^{\top}X^{+} \uvec_u}\right) X^{+}h^{\top}u^{\top} 
= X^{+}h^{\top}u^{\top} 
= X^{+} X^{\top +} \chi_{vu} \uvec_u^\top (I - X X^+)^\top .
\end{align*}
as only $\beta$ depends on $\epsilon$.
For the second term in \eqref{eq:4},
we compute
\begin{align*}
    \lim_{\epsilon\to0}\frac{\beta}{\sigma}pq^{\top}
    = k h
    = X^+ \uvec_u \chi_{vu}^\top X^+
\end{align*}
since $\frac{\beta}{\sigma} \to 1$, $p \to -k$ and $q^\top \to -h$ for $\epsilon \to 0$.
We plug this back into \eqref{eq:4} to
obtain that overall,
\[
    \frac{\partial X^+}{\partial X_{uv}}
    = X^{+} X^{\top +} \chi_{vu} \uvec_u^\top (I - X X^+)^\top  -
        X^+ \uvec_u \chi_{vu}^\top X^+ .
\]
The final statement follows as $X X^+$ is symmetric.
\end{proof}

\begin{corollary}
\label{cor:lapl-jac}
For any $u \not= v$,
\[
    \left(\frac{\partial L}{\partial L_{uv}^{+}}\right) L^+ \ones
    = LL^{\top}\chi_{uv}\uvec_u^{\top}\ones .
\]
\end{corollary}

\begin{proof}
By properties of the pseudoinverse,
$(I - L^+ L) L^+ \ones = L^+ \ones - L^+ L L^+ \ones = L^+ \ones - L^+ \ones = \zeros$.
We use this to simplify the Jacobian obtained in Lemma~\ref{lem:lapl-jac}:
\begin{align*}
    \left(\frac{\partial L}{\partial L_{uv}^{+}}\right) L^+ \ones
    = LL^{\top}\chi_{vu}\uvec_u^{\top}\left(I-L^{+}L\right) L^{+}\ones
    -L\uvec_u\chi_{vu}^{\top}LL^{+}\ones
    = L\uvec_u\chi_{uv}^{\top}LL^{+}\ones .
\end{align*}
\end{proof}

\section{Additional Experimental Results}
\label{sec:apx-exp}

We present the experimental results that we omitted from the main body due to space constraints.

Figure~\ref{fig:ht-inf} shows the
convergence behavior of our gradient descent approach
for learning a single Markov chain. We show
the recovery error for a fixed number of iterations
and the recovery error per iteration, compared
to the recovery error achieved by
\citet{wittmann2009reconstruction}. As we observe in Figure~\ref{fig:ht-inf} (left), the number of iterations should depend on the number of states. We observe a sudden transition in the loss function that translates to a sudden transition to the recovery error as well from 450 to 500 states.  For $n=25$ states, we observe  in Figure~\ref{fig:ht-inf} (right) that a few thousand iterations suffice to obtain a good solution. After 1000 iterations our obtained solution always improves the solution found by Wittman et al.~\cite{wittmann2009reconstruction} and occassionally even 500 iterations suffice. 

Figures~\ref{fig:ht-inf-from-trails} and \ref{fig:ht-inf-noise-hetero} show the recovery error of our approach and the approach of \citet{wittmann2009reconstruction}, for which we reported the difference in Figure~\ref{fig:ht-inf-noise}. We also show results with heteroscedastic noise in Figure~\ref{fig:ht-inf-noise-hetero}.
In this scenario, the noise level $\sigma_{uv}$ is proportional to $\sigma(H(u,v))$, with $\sigma_{uv} = 2 \cdot H(u,v) / t_{\text{cover}}$, where $t_{\text{cover}}$ is the cover time of the underlying Markov chain.
The observed trend continues to be approximately linear, mirroring the pattern observed with homoscedastic noise.

\begin{figure}
    \centering
    \includegraphics[width=0.5\columnwidth]{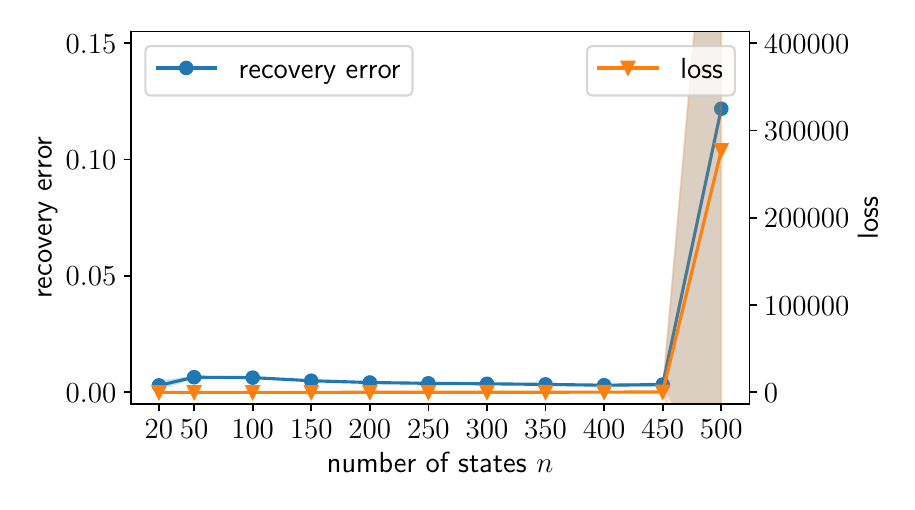}~
    \includegraphics[width=0.5\columnwidth]{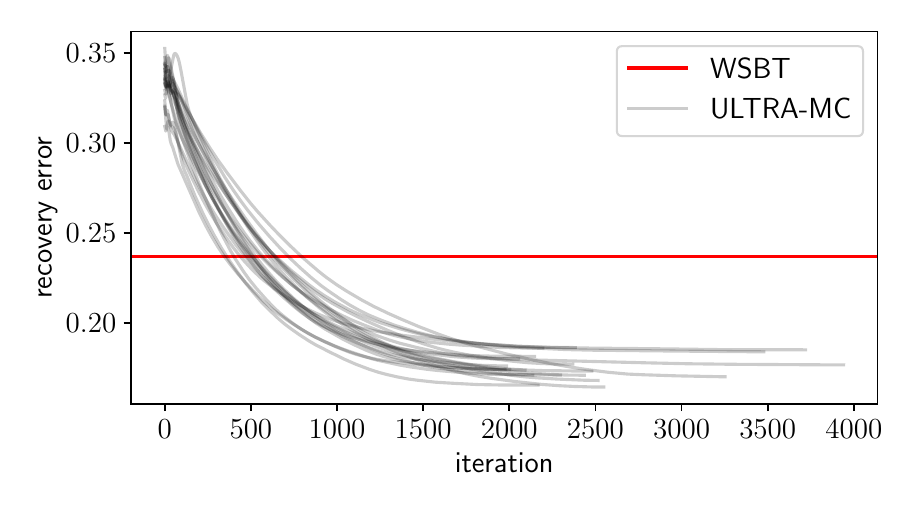}
    \vspace{-2em}
    \caption{ (Left) Learning a single transition matrix $M$ from hitting times using $10\,000$ gradient descent
    iterations. (Right) Convergence of \ULTRAMC from a random initialization compared with the method of \citet{wittmann2009reconstruction} (WSBT), for a uniform random Markov chain on $n=25$ states and homoscedastic noise with $\sigma=0.5$.}
    \label{fig:ht-inf}
\end{figure}

\begin{figure}
    \centering
    \includegraphics[width=0.5\columnwidth]{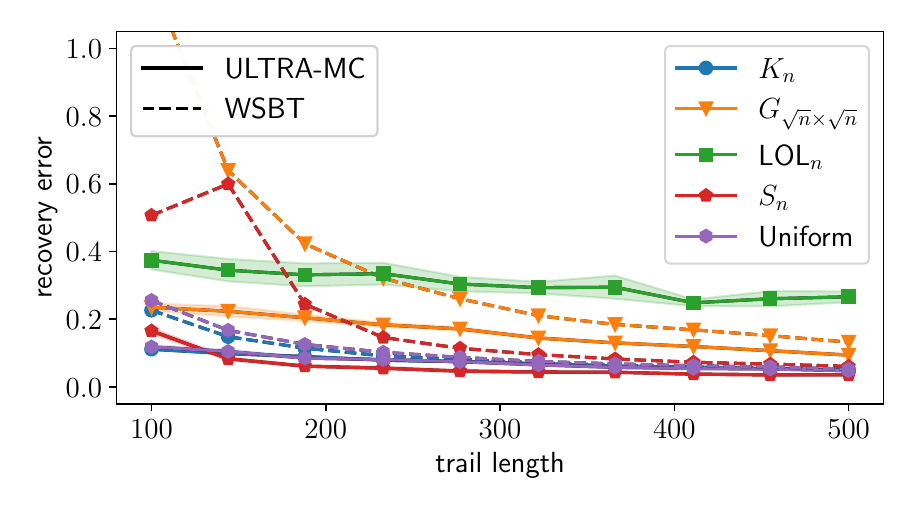}
    \vspace{-1em}
    \caption{Learning Markov chains via hitting times estimated from trails.}
    \label{fig:ht-inf-from-trails}
\end{figure}

\begin{figure}
    \centering
    \includegraphics[width=0.5\columnwidth]{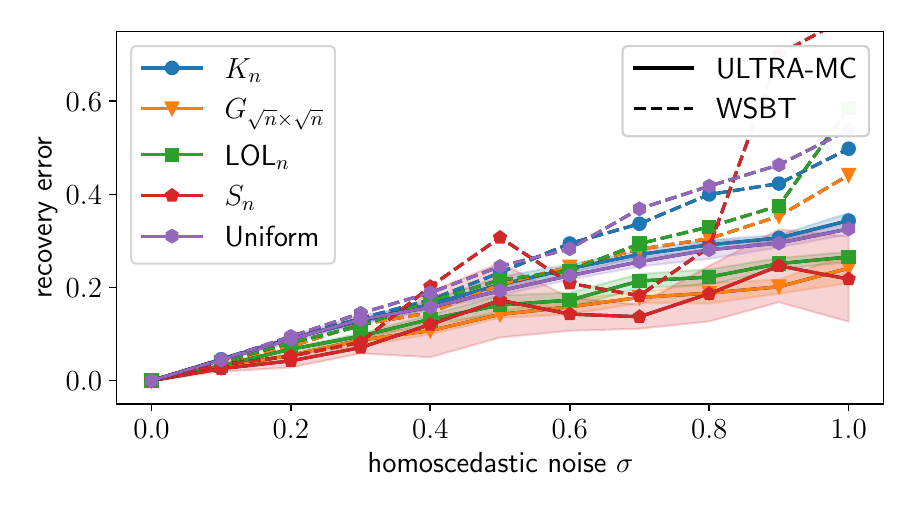}~
    \includegraphics[width=0.5\columnwidth]{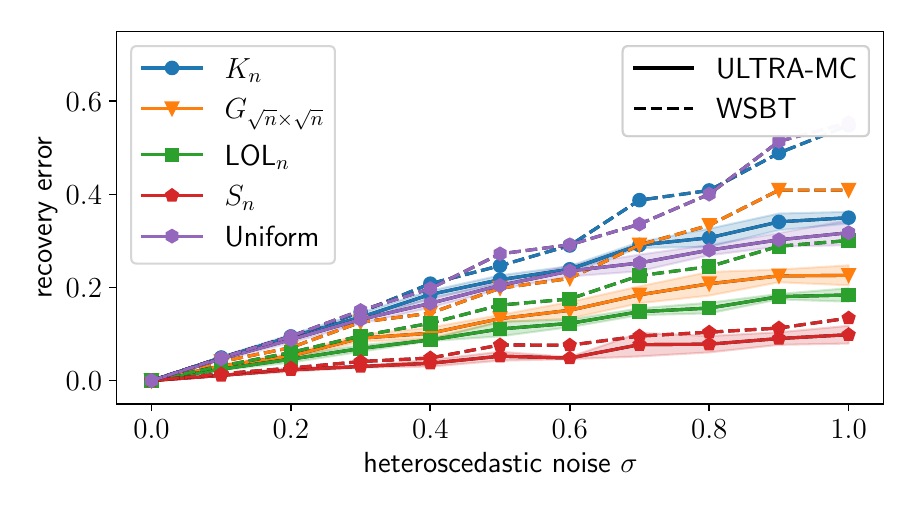}
    \vspace{-1em}
    \caption{Learning Markov chains under homoscedastic and heteroscedastic noise.}
    \label{fig:ht-inf-noise-hetero}
\end{figure}

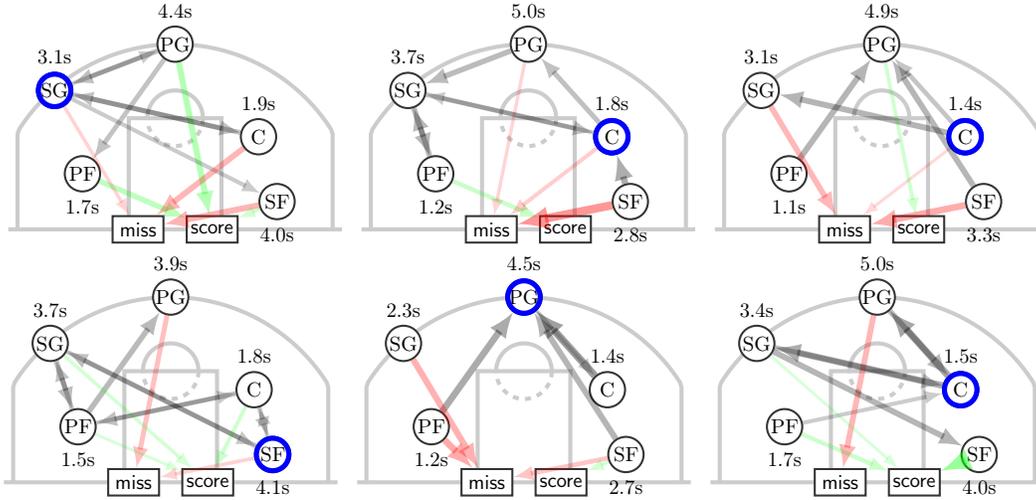
\begin{figure}
    \centering

\tikzstyle{court}=[black!20, line width=2pt,samples=100]
\tikzstyle{player}=[circle, inner sep=0pt, text width=16pt, align=center, draw=black!80, line width=1pt, fill=white]
\tikzstyle{pass}=[-latex, black]
\tikzstyle{basket}=[player, rectangle, inner sep=2pt, minimum height=13pt, text centered, text width=21pt]
\tikzstyle{start}=[player,fill=none,draw=blue,line width=2.5pt]

\scalebox{0.77}{
\begin{tikzpicture}[scale=1.6]

    \draw[court, domain=0:180] plot ({min(1.7, max(-1.7, 2*cos(\x)))}, {2*sin(\x)});
    \draw[court] (-0.5,0) -- (-0.5,1.2) -- (0.5,1.2) -- (0.5,0);
    \draw[court, domain=0:180] plot ({0.3*cos(\x)}, {1.2 + 0.3*sin(\x)});
    \draw[court, dashed, domain=0:180] plot ({0.3*cos(\x)}, {1.2 - 0.3*sin(\x)});
    \draw[court] (-1.8,0) -- (1.8,0);

    \node[player] (PG) at (0,2) {PG};
    \node[player] (SG) at (-1.3,1.5) {SG};
    \node[player] (PF) at (-1.0,0.6) {PF};
    \node[player] (C) at (0.9,1.0) {C};
    \node[player] (SF) at (1.1,0.3) {SF};

    \node[basket] (miss) at (-0.4,0) {$\mathsf{miss}$};
    \node[basket] (score) at (0.4,0) {$\mathsf{score}$};

    \node[label={[label distance=6pt]below:{1.7s}}] at (PF) {};
    \draw[pass,opacity=0.26510165592858015,line width=2.6510165592858015pt,green] (PF) to (score);
    \node[label={[label distance=6pt]above:{4.4s}}] at (PG) {};
    \draw[pass,opacity=0.29458084122517175,line width=2.9458084122517176pt,green] (PG) to (score);
    \draw[pass,opacity=0.2383308418504324,line width=2.383308418504324pt] (PG) to (PF);
    \draw[pass,opacity=0.24858367497414616,line width=2.4858367497414617pt] (PG) to (SG);
    \node[label={[label distance=6pt]above:{3.1s}}] at (SG) {};
    \node[start] at (SG) {};
    \draw[pass,opacity=0.17560673890701559,line width=1.756067389070156pt,red] (SG) to (miss);
    \draw[pass,opacity=0.2061996884758428,line width=2.061996884758428pt] (SG) to (PG);
    \draw[pass,opacity=0.24610954646880243,line width=2.4610954646880243pt] (SG) to (C);
    \draw[pass,opacity=0.2053183704136348,line width=2.053183704136348pt] (SG) to (SF);
    \node[label={[label distance=6pt]above:{1.9s}}] at (C) {};
    \draw[pass,opacity=0.28952125376435706,line width=2.8952125376435704pt,red] (C) to (miss);
    \draw[pass,opacity=0.24088936968928049,line width=2.4088936968928047pt] (C) to (SG);
    \node[label={[label distance=6pt]below:{4.0s}}] at (SF) {};
    \draw[pass,opacity=0.29716807517987137,line width=2.9716807517987136pt,red] (SF) to (miss);
    \draw[pass,opacity=0.1777627906759324,line width=1.777627906759324pt,green] (SF) to (score);

\end{tikzpicture}}
\scalebox{0.77}{
\begin{tikzpicture}[scale=1.6]

    \draw[court, domain=0:180] plot ({min(1.7, max(-1.7, 2*cos(\x)))}, {2*sin(\x)});
    \draw[court] (-0.5,0) -- (-0.5,1.2) -- (0.5,1.2) -- (0.5,0);
    \draw[court, domain=0:180] plot ({0.3*cos(\x)}, {1.2 + 0.3*sin(\x)});
    \draw[court, dashed, domain=0:180] plot ({0.3*cos(\x)}, {1.2 - 0.3*sin(\x)});
    \draw[court] (-1.8,0) -- (1.8,0);

    \node[player] (PG) at (0,2) {PG};
    \node[player] (SG) at (-1.3,1.5) {SG};
    \node[player] (PF) at (-1.0,0.6) {PF};
    \node[player] (C) at (0.9,1.0) {C};
    \node[player] (SF) at (1.1,0.3) {SF};

    \node[basket] (miss) at (-0.4,0) {$\mathsf{miss}$};
    \node[basket] (score) at (0.4,0) {$\mathsf{score}$};
    
    \node[label={[label distance=6pt]below:{1.2s}}] at (PF) {};
    \draw[pass,opacity=0.21872581026476057,line width=2.187258102647606pt,green] (PF) to (score);
    \draw[pass,opacity=0.27182768351502307,line width=2.7182768351502307pt] (PF) to (SG);
    \node[label={[label distance=6pt]above:{5.0s}}] at (PG) {};
    \draw[pass,opacity=0.1713026259850674,line width=1.713026259850674pt,red] (PG) to (miss);
    \draw[pass,opacity=0.26861837841123754,line width=2.6861837841123752pt] (PG) to (SG);
    \node[label={[label distance=6pt]above:{3.7s}}] at (SG) {};
    \draw[pass,opacity=0.27826255330459493,line width=2.782625533045949pt] (SG) to (PF);
    \draw[pass,opacity=0.24172062308035328,line width=2.417206230803533pt] (SG) to (C);
    \node[label={[label distance=6pt]above:{1.8s}}] at (C) {};
    \node[start] at (C) {};
    \draw[pass,opacity=0.1979502828961021,line width=1.9795028289610208pt,red] (C) to (miss);
    \draw[pass,opacity=0.24499038804499915,line width=2.4499038804499915pt] (C) to (PG);
    \draw[pass,opacity=0.20631256318752442,line width=2.063125631875244pt] (C) to (SG);
    \node[label={[label distance=6pt]below:{2.8s}}] at (SF) {};
    \draw[pass,opacity=0.4104322982248672,line width=4.104322982248672pt,red] (SF) to (miss);
    \draw[pass,opacity=0.32170698119757596,line width=3.2170698119757595pt] (SF) to (C);

\end{tikzpicture}}
\scalebox{0.77}{
\begin{tikzpicture}[scale=1.6]

    \draw[court, domain=0:180] plot ({min(1.7, max(-1.7, 2*cos(\x)))}, {2*sin(\x)});
    \draw[court] (-0.5,0) -- (-0.5,1.2) -- (0.5,1.2) -- (0.5,0);
    \draw[court, domain=0:180] plot ({0.3*cos(\x)}, {1.2 + 0.3*sin(\x)});
    \draw[court, dashed, domain=0:180] plot ({0.3*cos(\x)}, {1.2 - 0.3*sin(\x)});
    \draw[court] (-1.8,0) -- (1.8,0);

    \node[player] (PG) at (0,2) {PG};
    \node[player] (SG) at (-1.3,1.5) {SG};
    \node[player] (PF) at (-1.0,0.6) {PF};
    \node[player] (C) at (0.9,1.0) {C};
    \node[player] (SF) at (1.1,0.3) {SF};

    \node[basket] (miss) at (-0.4,0) {$\mathsf{miss}$};
    \node[basket] (score) at (0.4,0) {$\mathsf{score}$};

    \node[label={[label distance=6pt]below:{1.1s}}] at (PF) {};
    \draw[pass,opacity=0.2909338937660928,line width=2.909338937660928pt] (PF) to (PG);
    \node[label={[label distance=6pt]above:{4.9s}}] at (PG) {};
    \draw[pass,opacity=0.1851329694306571,line width=1.8513296943065711pt,green] (PG) to (score);
    \node[label={[label distance=6pt]above:{3.1s}}] at (SG) {};
    \draw[pass,opacity=0.28320410778198774,line width=2.8320410778198775pt,red] (SG) to (miss);
    \node[label={[label distance=6pt]above:{1.4s}}] at (C) {};
    \node[start] at (C) {};
    \draw[pass,opacity=0.1528596400039857,line width=1.5285964000398569pt,red] (C) to (miss);
    \draw[pass,opacity=0.22885141602747153,line width=2.2885141602747154pt] (C) to (PG);
    \draw[pass,opacity=0.2916227342095075,line width=2.916227342095075pt] (C) to (SG);
    \node[label={[label distance=6pt]below:{3.3s}}] at (SF) {};
    \draw[pass,opacity=0.3264624381685462,line width=3.264624381685462pt,red] (SF) to (miss);
    \draw[pass,opacity=0.2866733251311874,line width=2.8667332513118735pt] (SF) to (PG);

\end{tikzpicture}} \\
\scalebox{0.77}{
\begin{tikzpicture}[scale=1.6]

    \draw[court, domain=0:180] plot ({min(1.7, max(-1.7, 2*cos(\x)))}, {2*sin(\x)});
    \draw[court] (-0.5,0) -- (-0.5,1.2) -- (0.5,1.2) -- (0.5,0);
    \draw[court, domain=0:180] plot ({0.3*cos(\x)}, {1.2 + 0.3*sin(\x)});
    \draw[court, dashed, domain=0:180] plot ({0.3*cos(\x)}, {1.2 - 0.3*sin(\x)});
    \draw[court] (-1.8,0) -- (1.8,0);

    \node[player] (PG) at (0,2) {PG};
    \node[player] (SG) at (-1.3,1.5) {SG};
    \node[player] (PF) at (-1.0,0.6) {PF};
    \node[player] (C) at (0.9,1.0) {C};
    \node[player] (SF) at (1.1,0.3) {SF};

    \node[basket] (miss) at (-0.4,0) {$\mathsf{miss}$};
    \node[basket] (score) at (0.4,0) {$\mathsf{score}$};

    \node[label={[label distance=6pt]below:{1.5s}}] at (PF) {};
    \draw[pass,opacity=0.15380358355569995,line width=1.5380358355569994pt,green] (PF) to (score);
    \draw[pass,opacity=0.27440529374207384,line width=2.7440529374207383pt] (PF) to (PG);
    \draw[pass,opacity=0.23335037525701952,line width=2.333503752570195pt] (PF) to (SG);
    \draw[pass,opacity=0.20549686766535183,line width=2.0549686766535182pt] (PF) to (C);
    \node[label={[label distance=6pt]above:{3.9s}}] at (PG) {};
    \draw[pass,opacity=0.24825729006937464,line width=2.4825729006937465pt,red] (PG) to (miss);
    \node[label={[label distance=6pt]above:{3.7s}}] at (SG) {};
    \draw[pass,opacity=0.17551650846024403,line width=1.7551650846024403pt,green] (SG) to (score);
    \draw[pass,opacity=0.2621192301045261,line width=2.621192301045261pt] (SG) to (PF);
    \draw[pass,opacity=0.2232303607063736,line width=2.232303607063736pt] (SG) to (SF);
    \node[label={[label distance=6pt]above:{1.8s}}] at (C) {};
    \draw[pass,opacity=0.19514111524341507,line width=1.9514111524341506pt,green] (C) to (score);
    \draw[pass,opacity=0.24370855982881215,line width=2.4370855982881214pt] (C) to (PF);
    \draw[pass,opacity=0.2287889313775109,line width=2.2878893137751093pt] (C) to (SF);
    \node[label={[label distance=6pt]below:{4.1s}}] at (SF) {};
    \node[start] at (SF) {};
    \draw[pass,opacity=0.1799599980400186,line width=1.799599980400186pt,red] (SF) to (miss);
    \draw[pass,opacity=0.2486343574605205,line width=2.486343574605205pt] (SF) to (SG);
    \draw[pass,opacity=0.20355699599760024,line width=2.0355699599760024pt] (SF) to (C);

\end{tikzpicture}}
\scalebox{0.77}{
\begin{tikzpicture}[scale=1.6]

    \draw[court, domain=0:180] plot ({min(1.7, max(-1.7, 2*cos(\x)))}, {2*sin(\x)});
    \draw[court] (-0.5,0) -- (-0.5,1.2) -- (0.5,1.2) -- (0.5,0);
    \draw[court, domain=0:180] plot ({0.3*cos(\x)}, {1.2 + 0.3*sin(\x)});
    \draw[court, dashed, domain=0:180] plot ({0.3*cos(\x)}, {1.2 - 0.3*sin(\x)});
    \draw[court] (-1.8,0) -- (1.8,0);

    \node[player] (PG) at (0,2) {PG};
    \node[player] (SG) at (-1.3,1.5) {SG};
    \node[player] (PF) at (-1.0,0.6) {PF};
    \node[player] (C) at (0.9,1.0) {C};
    \node[player] (SF) at (1.1,0.3) {SF};

    \node[basket] (miss) at (-0.4,0) {$\mathsf{miss}$};
    \node[basket] (score) at (0.4,0) {$\mathsf{score}$};

    \node[label={[label distance=6pt]below:{1.2s}}] at (PF) {};
    \draw[pass,opacity=0.33118212732779806,line width=3.311821273277981pt,red] (PF) to (miss);
    \draw[pass,opacity=0.3295939747949673,line width=3.2959397479496735pt] (PF) to (PG);
    \node[label={[label distance=6pt]above:{4.5s}}] at (PG) {};
    \node[start] at (PG) {};
    \draw[pass,opacity=0.202714863880485,line width=2.0271486388048503pt] (PG) to (C);
    \node[label={[label distance=6pt]above:{2.3s}}] at (SG) {};
    \draw[pass,opacity=0.3087138047275384,line width=3.087138047275384pt,red] (SG) to (miss);
    \node[label={[label distance=6pt]above:{1.4s}}] at (C) {};
    \draw[pass,opacity=0.37296857692760854,line width=3.7296857692760854pt] (C) to (PG);
    \node[label={[label distance=6pt]below:{2.7s}}] at (SF) {};
    \draw[pass,opacity=0.2308268783523281,line width=2.308268783523281pt,red] (SF) to (miss);
    \draw[pass,opacity=0.17484054607101848,line width=1.7484054607101849pt,green] (SF) to (score);
    \draw[pass,opacity=0.29361882144228807,line width=2.936188214422881pt] (SF) to (PG);
    
\end{tikzpicture}}
\scalebox{0.77}{
\begin{tikzpicture}[scale=1.6]

    \draw[court, domain=0:180] plot ({min(1.7, max(-1.7, 2*cos(\x)))}, {2*sin(\x)});
    \draw[court] (-0.5,0) -- (-0.5,1.2) -- (0.5,1.2) -- (0.5,0);
    \draw[court, domain=0:180] plot ({0.3*cos(\x)}, {1.2 + 0.3*sin(\x)});
    \draw[court, dashed, domain=0:180] plot ({0.3*cos(\x)}, {1.2 - 0.3*sin(\x)});
    \draw[court] (-1.8,0) -- (1.8,0);

    \node[player] (PG) at (0,2) {PG};
    \node[player] (SG) at (-1.3,1.5) {SG};
    \node[player] (PF) at (-1.0,0.6) {PF};
    \node[player] (C) at (0.9,1.0) {C};
    \node[player] (SF) at (1.1,0.3) {SF};

    \node[basket] (miss) at (-0.4,0) {$\mathsf{miss}$};
    \node[basket] (score) at (0.4,0) {$\mathsf{score}$};

    \node[label={[label distance=6pt]below:{1.7s}}] at (PF) {};
    \draw[pass,opacity=0.20847605885349677,line width=2.0847605885349676pt,green] (PF) to (score);
    \draw[pass,opacity=0.2001092564635487,line width=2.001092564635487pt] (PF) to (C);
    \node[label={[label distance=6pt]above:{5.0s}}] at (PG) {};
    \draw[pass,opacity=0.2845671598026733,line width=2.845671598026733pt,red] (PG) to (miss);
    \draw[pass,opacity=0.2743386814650957,line width=2.743386814650957pt] (PG) to (C);
    \node[label={[label distance=6pt]above:{3.4s}}] at (SG) {};
    \draw[pass,opacity=0.15267372181655098,line width=1.52673721816551pt,green] (SG) to (score);
    \draw[pass,opacity=0.2463312684467804,line width=2.463312684467804pt] (SG) to (C);
    \draw[pass,opacity=0.28734556401841704,line width=2.8734556401841704pt] (SG) to (SF);
    \node[label={[label distance=6pt]above:{1.5s}}] at (C) {};
    \node[start] at (C) {};
    \draw[pass,opacity=0.3468759216357968,line width=3.468759216357968pt] (C) to (PG);
    \draw[pass,opacity=0.32393812144070355,line width=3.2393812144070355pt] (C) to (SG);
    \node[label={[label distance=6pt]below:{4.0s}}] at (SF) {};
    \draw[pass,opacity=0.4009730291786204,line width=4.009730291786203pt,green] (SF) to (score);
    
\end{tikzpicture}}
\caption{\label{fig:boston} Offensive strategies of the Boston Celtics during the 2022 season, learned from a mixture of $C=6$ continuous-time Markov chains.}
\end{figure}




\begin{table}[h]
\caption{Empirical standard deviation of the results in Figure~\ref{fig:ht-mix}}
\label{tab:std-ht-mix}
\begin{center}
\begin{small}
\begin{tabular}{r|llllllll}
\toprule

$n$ &
3 &
4 &
5 &
6 &
7 &
8 &
9 &
10 \\
\midrule

ULTRA-MC &
0.04 & 0.0 & 0.11 & 0.0 & 0.0 & 0.0 & 0.15 & 0.08 \\
EM (continuous) &
0.08 & 0.06 & 0.15 & 0.08 & 0.0 & 0.0 & \\
KTT (discretized) &
0.03 & 0.04 & 0.03 & 0.03 & 0.02 & 0.01 & \\
EM (discretized) &
0.18 & 0.2 & 0.19 & 0.18 & 0.13 & 0.17 & 0.14 & 0.0 \\
SVD (discretized) &
0.16 & 0.15 & 0.13 & 0.13 & 0.16 & 0.08 & 0.13 & 0.07 \\

\bottomrule
\end{tabular}
\end{small}
\end{center}
\end{table}

\end{document}